%% file: main_critview.tex
\documentclass{article} 
\usepackage{iclr2019_conference,times}

\usepackage{hyperref}
\usepackage{url}
\definecolor{darkblue}{rgb}{0,0.0,0.55}
\hypersetup{ 
	colorlinks = true,
	citecolor  = darkblue,
	linkcolor  = darkblue,
	citecolor  = darkblue,
	filecolor  = darkblue,
	urlcolor   = darkblue,
}

\usepackage{rf}

\newcommand{\nfro}[1]{\|{#1}\|_{\rm F}} 
\newcommand{\sigmin}{\sigma_{\rm{min}}}
\newcommand{\sigmax}{\sigma_{\rm{max}}}
\newcommand{\rowsp}{\mathop{\rm row}}
\newcommand{\colsp}{\mathop{\rm col}}
\newcommand{\nulsp}{\mathop{\rm null}}
\newcommand{\lnulsp}{\mathop{\rm leftnull}}
\providecommand{\tr}{\mathop{\rm tr}}
\providecommand{\rank}{\mathop{\rm rank}}

\newcommand{\matent}[3]{[{#1}]_{{#2},{#3}}}
\newcommand{\matcol}[2]{[{#1}]_{\cdot,{#2}}}
\newcommand{\veccomp}[2]{[{#1}]_{#2}}
\newcommand{\dotprod}[2]{\langle{#1},{#2}\rangle}
\newcommand\numberthis{\addtocounter{equation}{1}\tag{\theequation}}
\newcommand{\gradlzero}{\nabla \ell_0}
\newcommand{\ones}[1]{\mathbf{1}_{#1}}
\newcommand{\zeros}[1]{\mathbf{0}_{#1}}
\newcommand{\bdindexset}{\mc J}
\newcommand{\relulikeactfun}{\bar h_{s_+,s_-}}
\newcommand{\tuple}{(W_j)_{j=1}^{H+1}}

\usepackage{accents}
\usepackage{enumitem}

\title{Small nonlinearities in activation functions create bad local minima in neural networks}

\author{Chulhee Yun, Suvrit Sra \& Ali Jadbabaie \\
Massachusetts Institute of Technology\\
Cambridge, MA 02139, USA\\
\texttt{\{chulheey,suvrit,jadbabai\}@mit.edu}
}

\iclrfinalcopy 
\begin{document}

\maketitle

\begin{abstract}
We investigate the loss surface of neural networks. We prove that even for one-hidden-layer networks with ``slightest'' nonlinearity, the empirical risks have spurious local minima in most cases. Our results thus indicate that in general ``\emph{no spurious local minima}'' is a property limited to deep linear networks, and insights obtained from linear networks may not be robust. Specifically, for ReLU(-like) networks we constructively prove that for \emph{almost all} practical datasets there exist infinitely many local minima. We also present a counterexample for more general activations (sigmoid, tanh, arctan, ReLU, etc.), for which there exists a bad local minimum. Our results make the least restrictive assumptions relative to existing results on spurious local optima in neural networks. We complete our discussion by presenting a comprehensive characterization of global optimality for deep linear networks, which unifies other results on this topic.
\end{abstract}

\input{intro}
\input{nonlin}
\input{linear}

\vspace*{-5pt}
\section{Discussion and future work}
\vspace*{-5pt}
We investigated the loss surface of deep linear and nonlinear neural networks. We proved two theorems showing existence of spurious local minima on nonlinear networks, which apply to almost all datasets (Theorem~\ref{thm:piecelin}) and a wide class of activations (Theorem~\ref{thm:othernonlin}).
We concluded by Theorem~\ref{thm:linear}, showing a general result studying the behavior of critical points in multilinearly parametrized functions, which unifies other existing results on linear neural networks.
Given that spurious local minima are common in neural networks, a valuable future research direction will be investigating how far local minima are from global minima in general, and how the size of the network affects this gap. Another thing to note is that even though we showed the existence of spurious local minima in the \emph{whole} parameter space, things can be different in restricted sets of parameter space (e.g., by adding regularizers). Understanding the loss surface in such sets would be valuable. Additionally, one can try to show algorithmic/trajectory results of (stochastic) gradient descent. We hope that our paper will be a stepping stone to such future research.

\subsubsection*{Acknowledgments}
This work was supported by the DARPA Lagrange Program. Suvrit Sra also acknowledges support from an Amazon Research Award.

\bibliography{cite}
\bibliographystyle{iclr2019_conference}

\appendix
\newpage

\setcounter{theorem}{0}
\setcounter{equation}{0}
\renewcommand\thesection{A\arabic{section}}
\renewcommand{\thetheorem}{A.\arabic{theorem}}
\renewcommand{\theequation}{A.\arabic{equation}}

\allowdisplaybreaks

\input{apxlin}

\end{document}

%% file: intro.tex
\vspace*{-6pt}
\section{Introduction}
\vspace*{-6pt}
Neural network training reduces to solving nonconvex empirical risk minimization problems, a task that is in general intractable. But success stories of deep learning suggest that local minima of the empirical risk could be close to global minima. \citet{choromanska2015loss} use spherical spin-glass models from statistical physics to justify how the size of neural networks may result in local minima that are close to global. However, due to the complexities introduced by nonlinearity, a rigorous understanding of optimality in deep neural networks remains elusive.

Initial steps towards understanding optimality have focused on \emph{deep linear} networks. This area has seen substantial recent progress. In deep linear networks there is no nonlinear activation; the output is simply a multilinear function of the input. \citet{baldi1989neural} prove that some shallow  networks have no spurious local minima, and \citet{kawaguchi2016deep} extends this result to squared error deep linear networks, showing that they only have global minima and saddle points. Several other works on linear nets have also appeared~\citep{lu2017depth,freeman2017topology,yun2018global,zhou2018critical,laurent2017multilinear,laurent2018deep}.


The theory of nonlinear neural networks (which is the actual setting of interest), however, is still in its infancy. There have been attempts to extend the ``local minima are global'' property from linear to nonlinear networks, but recent results suggest that this property does not usually hold~\citep{zhou2018critical}. Although not unexpected, rigorously proving such results turns out to be non-trivial, forcing several authors (e.g., \citet{safran2017spurious,du2017gradient,wu2018no}) to make somewhat unrealistic assumptions (realizability and Gaussianity) on data.


In contrast, we prove existence of spurious local minima under the least restrictive (to our knowledge) assumptions. Since seemingly subtle changes to assumptions can greatly influence the analysis as well as the applicability of known results, let us first summarize what is known; this will also help provide a better intuitive perspective on our results (as the technical details are somewhat involved).

\vspace*{-6pt}
\subsection{What is known so far?}
\vspace*{-6pt}
There is a large and rapidly expanding literature of optimization of neural networks. Some works focus on the loss surface \citep{baldi1989neural, yu1995local, kawaguchi2016deep, swirszcz2016local, soudry2016no, xie2016diverse, nguyen2017loss, nguyen2017losscnn, safran2017spurious, laurent2017multilinear, yun2018global, zhou2018critical,  wu2018no, liang2018adding, liang2018understanding, shamir2018resnets}, while others study the convergence of gradient-based methods for optimizing this loss \citep{tian2017analytical, brutzkus2017globally, zhong2017recovery, soltanolkotabi2017learning, li2017convergence, du2017gradient, zhang2018learning, brutzkus2018sgd, wang2018learning, li2018learning, du2018gradientB, du2018gradientA, allen2018convergence, zou2018stochastic, zhou2019sgd}. In particular, our focus is on the loss surface itself, independent of any algorithmic concerns; this is reflected in the works summarized below.

For ReLU networks, the works~\citep{swirszcz2016local,zhou2018critical} provide counterexample datasets that lead to spurious local minima, dashing hopes of ``local implies global'' properties. However, these works fail to provide statements about generic datasets, and one can argue that their setups are limited to isolated pathological examples. In comparison, our Theorem~\ref{thm:piecelin} shows existence of spurious local minima for \emph{almost all} datasets, a much more general result.
\citet{zhou2018critical} also give characterization of critical points of shallow ReLU networks, but with more than one hidden node the characterization provided is limited to certain regions.

There are also results that study population risk of shallow ReLU networks under an assumption that input data is i.i.d.\ Gaussian distributed \citep{safran2017spurious, wu2018no, du2017gradient}. Moreover, these works also assume \emph{realizability}, i.e., the output data is generated from a neural network with the same architecture as the model one trains, with unknown true parameters. These assumptions enable one to compute the population risk in a closed form, and ensure that one can always achieve zero loss at global minima. The authors of \citet{safran2017spurious, wu2018no} study the population risk function of the form $\E_x [(\sum_{i=1}^k \text{ReLU}(w_i^T x)-\text{ReLU}(v_i^Tx))^2]$, where the true parameters $v_i$'s are orthogonal unit vectors. Through extensive experiments and computer-assisted local minimality checks, \citet{safran2017spurious} show existence of local minima for $k \geq 6$. However, this result is empirical and does not have constructive proofs. \citet{wu2018no} show that with $k=2$, there is no bad local minima on the manifold $\ltwo{w_1} = \ltwo{w_2} = 1$. \citet{du2017gradient} study population risk of one-hidden-layer CNN. They show that there can be a spurious local minimum, but gradient descent converges to the global minimum with probability at least 1/4.

Our paper focuses on empirical risk instead of population risk, and \emph{does not}  assume either Gaussianity or realizability. Theorem~\ref{thm:piecelin}~1's assumption on the dataset is that it is \emph{not linearly fittable}\footnote{That is, given input data matrices $X$ and $Y$, there is no matrix $R$ such that $Y = RX$.}, which is vastly more general and realistic than assuming that input data is Gaussian or that the output is generated from an unknown neural network. Our results also show that \citet{wu2018no} fails to extend to empirical risk and non-unit parameter vectors (see the discussion after Theorem~\ref{thm:othernonlin}).

\citet{liang2018understanding} showed that under assumptions on the loss function, data distribution, network structure, and activation function, all local minima of the empirical loss have zero classification error in binary classification tasks. The result relies on stringent assumptions, and it is not directly comparable to ours because both ``the local minimum has nonzero classification error'' and ``the local minima is spurious'' do not imply one another. 
\citet{liang2018adding} proved that adding a parallel network with one exponential hidden node can eliminate all bad local minima. The result relies on the special parallel structure, whereas we analyze standard fully connected network architecture.

\citet{laurent2017multilinear} studies one-hidden-layer networks with hinge loss for classification. Under linear separability, the authors prove that Leaky-ReLU networks don't have bad local minima, while ReLU networks do. Our focus is on regression, and we only make mild assumptions on data.

For deep linear networks, the most relevant result to ours is~\citet{laurent2018deep}. When all hidden layers are wider than the input or  output layers, \citet{laurent2018deep} prove that any local minimum of a deep linear network under differentiable convex loss is global.\footnote{Although their result overlaps with a subset of Theorem~\ref{thm:linear}, our theorem was obtained independently.} They prove this by showing a statement about relationship between linear vs.\ multilinear parametrization. Our result in Theorem~\ref{thm:linear} is \emph{strictly} more general that their results, and presents a comprehensive characterization.

A different body of literature \citep{yu1995local, soudry2016no, xie2016diverse, nguyen2017loss, nguyen2017losscnn} 
considers sufficient conditions for global optimality in nonlinear networks. These results make certain architectural assumptions (and some technical restrictions) that may not usually apply to realistic networks.
There are also other works on global optimality conditions for specially designed architectures \citep{haeffele2017global, feizi2017porcupine}.

\vspace*{-6pt}
\subsection{Contributions and Summary of Results}
\vspace*{-6pt}
We summarize our key contributions more precisely below. Our work encompasses results for both nonlinear and linear neural networks. First, we study whether the ``local minima are global'' property holds for nonlinear networks. Unfortunately, our results here are negative. 
Specifically, we prove
\begin{list}{\small{$\blacktriangleright$}}{\leftmargin=1em}
  \setlength{\itemsep}{1pt}
  \vspace*{-6pt}
\item For piecewise linear and nonnegative homogeneous activation functions (e.g., ReLU), we prove in Theorem~\ref{thm:piecelin} that if linear models cannot perfectly fit the data, one can \emph{construct} infinitely many local minima that are not global. In practice, most datasets are not linearly fittable, hence this result gives a constructive proof of spurious local minima for generic datasets. In contrast, several existing results either provide only one counterexample~\citep{swirszcz2016local,zhou2018critical}, or make restrictive assumptions of realizability \citep{safran2017spurious, du2017gradient} or linear separability \citep{laurent2017multilinear}. This result is presented in Section~\ref{sec:piecelin}.
  
\item In Theorem~\ref{thm:othernonlin} we tackle more general nonlinear activation functions, and provide a simple architecture (with squared loss) and dataset, for which there exists a local minimum inferior to the global minimum for a realizable dataset. Our analysis applies to a wide range of activations, including sigmoid, tanh, arctan, ELU \citep{clevert2015fast}, SELU \citep{klambauer2017self}, and ReLU.
Considering that realizability of data simplifies the analysis and ensures zero loss at global optima, our counterexample that is realizable and yet has a spurious local minimum is surprising, suggesting that the situation is likely worse for non-realizable data. See Section~\ref{sec:othernonlin} for details.
\end{list}

We complement our negative results by presenting the following positive result on linear networks:
\begin{list}{\small{$\blacktriangleright$}}{\leftmargin=1em}
  \vspace*{-6pt}
\item Assume that the hidden layers are as wide as either the input or the output, and that the empirical risk $\ell(\tuple)$ equals $\ell_0(W_{H+1} W_H \cdots W_1)$, where $\ell_0$ is a differentiable loss function and $W_i$ is the weight matrix for layer $i$.
Theorem~\ref{thm:linear} shows if $(\hat W_j)_{j=1}^{H+1}$ is a critical point of $\ell$, then its type of stationarity (local min/max, or saddle) is closely related to the behavior of $\ell_0$ evaluated at the product $\hat W_{H+1} \cdots \hat W_1$. If we additionally assume that any critical point of $\ell_0$ is a global minimum, Corollary~\ref{cor:linearcvx} shows that the empirical risk $\ell$ only has global minima and saddles, and provides a simple condition to distinguish between them. To the best of our knowledge, this is the most general result on deep linear networks and it subsumes several previous results, e.g.,~\citep{kawaguchi2016deep, yun2018global,zhou2018critical,laurent2018deep}.
  This result is in Section~\ref{sec:linear}.
\end{list}

\vspace*{-6pt}
\paragraph{Notation.} For an integer $a \ge 1$, $[a]$ denotes the set of integers from $1$ to $a$ (inclusive). For a vector $v$, we use $\veccomp{v}{i}$ to denote its $i$-th component, while $\veccomp{v}{[i]}$ denotes a vector comprised of the first $i$ components of $v$.  
Let $\ones{(\cdot)}$ ($\zeros{(\cdot)}$) be the all ones (zeros) column vector or matrix with size $(\cdot)$.



%% file: nonlin.tex
\vspace*{-6pt}
\section{``ReLU-like'' networks: bad local minima exist for most data}
\vspace*{-5pt}
\label{sec:piecelin}
We study below whether nonlinear neural networks provably have spurious local minima. We show in~\S\ref{sec:piecelin} and \S\ref{sec:othernonlin} that even for extremely simple nonlinear networks, one encounters spurious local minima. We first consider ReLU and ReLU-like networks. Here, we prove that as long as linear models cannot perfectly fit the data, there exists a local minimum strictly inferior to the global one. Using nonnegative homogeneity, we can scale the parameters to get infinitely many local minima.

Consider a training dataset that consists of $m$ data points. The inputs and the outputs are of dimension $d_x$ and $d_y$, respectively. We aggregate these items, and write $X \in \reals^{d_x \times m}$ as the data matrix and $Y \in \reals^{d_y \times m}$ as the label matrix. Consider the 1-hidden-layer neural network 
$\hat Y = W_2 h (W_1 X + b_1 \ones{m}^T) + b_2 \ones{m}^T$,
where $h$ is a nonlinear activation function,
$W_2 \in \reals^{d_y \times d_1}$, $b_2 \in \reals^{d_y}$, $W_1 \in \reals^{d_1 \times d_x}$, and $b_1 \in \reals^{d_1}$.
We analyze the empirical risk with squared loss
\begin{align*}
\ell(W_1,W_2,b_1,b_2)
\!=\! \half \nfro{W_2 h(W_1 X \! + \! b_1 \ones{m}^T) \!+\!  b_2 \ones{m}^T \!-\! Y}^2. 
\end{align*}
Next, define a class of piecewise linear nonnegative homogeneous functions
\begin{equation}
\label{eq:3}
\relulikeactfun(x) = \max\{s_+ x,0\} + \min\{s_- x, 0\},
\end{equation}
where $s_+ > 0, s_- \geq 0$ and $s_+ \neq s_-$. 
Note that ReLU and Leaky-ReLU are members of this class.

\vspace*{-6pt}
\subsection{Main results and discussion}
\vspace*{-4pt}
We use the shorthand $\tilde X \defeq \begin{bmatrix} X^T & \ones{m} \end{bmatrix}^T \in \reals^{(d_x + 1) \times m}$. 
The main result of this section, Theorem~\ref{thm:piecelin}, considers the case where linear models cannot fit $Y$, i.e., $Y \neq R\tilde X$ for all matrix $R$. 
With ReLU-like activation~\eqref{eq:3} and a few mild assumptions, Theorem~\ref{thm:piecelin} shows that there exist spurious local minima.

\begin{theorem}
  \label{thm:piecelin}
  Suppose that the following conditions hold:
  \vspace*{-4pt}
  \begin{enumerate}[label={\small(C\ref{thm:piecelin}.\arabic*)}, leftmargin=30pt]
    \setlength{\itemsep}{0pt}
  \item Output dimension is $d_y = 1$, and linear models $R \tilde X$ cannot perfectly fit $Y$.
  \item All the data points $x_i$'s are distinct.
  \item The activation function $h$ is $\relulikeactfun$.
  \item The hidden layer has at least width 2: $d_1 \geq 2$. 
  \end{enumerate}
  Then, there is a spurious local minimum whose risk is the same as linear least squares model.  Moreover, due to nonnegative homogeneity of $\relulikeactfun$, there are infinitely many such local minima.
\end{theorem}
Noticing that most real world datasets cannot be perfectly fit with linear models,
Theorem~\ref{thm:piecelin} shows that when we use the activation $\relulikeactfun$, 
the empirical risk has bad local minima for \emph{almost all} datasets that one may encounter in practice. Although it is not very surprising that neural networks have spurious local minima, proving this rigorously is non-trivial. We provide a constructive and deterministic proof for this problem that holds for general datasets, which is in contrast to experimental results of \citet{safran2017spurious}. 
We emphasize that Theorem~\ref{thm:piecelin} also holds even for ``slightest'' nonlinearities, e.g., when $s_+ = 1+\epsilon$ and $s_- = 1$ where $\epsilon > 0$ is small. This suggests that the ``local min is global'' property is limited to the simplified setting of \emph{linear} neural networks.

Existing results on squared error loss either provide one counterexample \citep{swirszcz2016local, zhou2018critical}, or assume realizability and Gaussian input \citep{safran2017spurious, du2017gradient}.
Realizability is an assumption that the output is generated by a network with unknown parameters.
In real datasets, neither input is Gaussian nor output is generated by neural networks; in contrast, our result holds for most realistic situations, and hence delivers useful insight.

There are several results proving sufficient conditions for global optimality of nonlinear neural networks~\citep{soudry2016no, xie2016diverse, nguyen2017loss}. But they rely on assumptions that the network width scales with the number of data points. For instance, applying Theorem 3.4 of \citet{nguyen2017loss} to our network proves that if $\tilde X$ has linearly independent columns and other assumptions hold, then any critical point with $W_2 \neq 0$ is a global minimum.
However, linearly independent columns already imply $\rowsp (\tilde X) = \reals^{m}$, 
so even linear models $R \tilde X$ can fit any $Y$; i.e., there is less merit in using a complex model to fit $Y$. Theorem~\ref{thm:piecelin} does not make any structural assumption other than $d_1 \geq 2$, and addresses the case where it is \emph{impossible} to fit $Y$ with linear models, which is much more realistic. 

It is worth comparing our result with \citet{laurent2017multilinear}, who use hinge loss based classification and assume linear separability to prove ``no spurious local minima'' for Leaky-ReLU networks. Their result does not contradict our theorem because the losses are different and we do not assume linear separability. 

One might wonder if our theorem holds even with $d_1 \geq m$. 
\citet{venturi2018neural} showed that one-hidden-layer neural networks with $d_1 \geq m$ doesn't have spurious valleys, hence there is no \emph{strict} spurious local minima; however, due to nonnegative homogeneity of $\relulikeactfun$ we only have non-strict local minima. 
Based on \citet{bengio2006convex}, one might claim that with wide enough hidden layer and random $W_1$ and $b_1$, one can fit any $Y$; however, this is not the case, by our assumption that linear models $R \tilde X$ cannot fit $Y$. 
Note that for any $d_1$, there is a non-trivial region (measure $> 0$) in the parameter space where $W_1 X + b_1 \ones{m}^T > \zeros{}$ (entry-wise). 
In this region, the output of neural network $\hat Y$ is still a linear combination of rows of $\tilde X$, so $\hat Y$ cannot fit $Y$; in fact, it can only do as well as linear models. We will see in the Step~1 of Section~\ref{sec:thmpiecelinproof} that the bad local minimum that we construct ``kills'' $d_1 - 1$ neurons; however, killing many neurons is not a necessity, and it is just to simply the exposition. In fact, any local minimum in the region $W_1 X + b_1 \ones{m}^T > \zeros{}$ is a spurious local minimum.

\vspace*{-5pt}
\subsection{Analysis of Theorem~\ref{thm:piecelin}}
\label{sec:thmpiecelinproof}
\vspace*{-4pt}
The proof of the theorem is split into two steps. First, we prove that there exist local minima $(\hat W_j, \hat b_j)_{j=1}^2$ whose risk value is the same as the linear least squares solution, and that there are infinitely many such minima. Second, we will construct a tuple of parameters $(\tilde W_j, \tilde b_j)_{j=1}^2$ that has strictly smaller empirical risk than $(\hat W_j, \hat b_j)_{j=1}^2$.

\textbf{Step 1: A local minimum as good as the linear solution.}
The main idea here is to exploit the weights from the linear least squares solution, and to tune the parameters so that all inputs to hidden nodes become positive. Doing so makes the hidden nodes ``locally linear,'' so that the constructed $(\hat W_j, \hat b_j)_{j=1}^2$ that
produce linear least squares estimates at the output become locally optimal.

Recall that
$\tilde X = \begin{bmatrix} X^T & \ones{m} \end{bmatrix}^T \in \reals^{(d_x + 1) \times m}$, and 
define a linear least squares loss $\ell_0 (R) \defeq \half \nfro{R \tilde X - Y}^2$ that is minimized at $\bar W$, so that $\nabla\ell_0(\bar W) = (\bar W \tilde X - Y) \tilde X^T = 0$.
Since $d_y = 1$, the solution $\bar W \in \reals^{d_y \times (d_x+1)}$ is a row vector.
For all $i \in [m]$, let $\bar y_i = \bar W \begin{bmatrix} x_i^T & 1\end{bmatrix}^T$ be the output of the linear least squares model, and similarly $\bar Y = \bar W \tilde X$.

Let $\eta \defeq \min \left \{ -1, 2 \min_i \bar y_i \right \}$, a negative constant making $\bar y_i - \eta > 0$ for all $i$.
Define parameters
\begin{align*}
\hat W_1 = \alpha
\begin{bmatrix}
\veccomp{\bar W}{[d_x]} \\
\zeros{(d_1-1) \times d_x}
\end{bmatrix}
,~
\hat b_1 = \alpha
\begin{bmatrix}
\veccomp{\bar W}{d_x+1}-\eta\\
-\eta \ones{d_1-1}
\end{bmatrix}
,~
\hat W_2 = \begin{bmatrix} \frac{1}{\alpha s_+} & \zeros{d_1-1}^T \end{bmatrix}
, ~
\hat b_2 = \eta,
\end{align*}
where $\alpha > 0$ is any arbitrary fixed positive constant, $\veccomp{\bar W}{[d_x]}$ gives the first $d_x$ components of $\bar W$, and $\veccomp{\bar W}{d_x+1}$ the last component.
Since $\bar y_i = \veccomp{\bar W}{[d_x]} x_i + \veccomp{\bar W}{d_x+1}$, for any $i$, $\hat W_1 x_i + \hat b_1 > \zeros{d_1}$ (component-wise), given
our choice of $\eta$. Thus, all hidden node inputs are positive. Moreover,
$\hat Y 
= \tfrac{1}{\alpha s_+} s_+ (\alpha \bar Y - \alpha \eta \ones{m}^T ) + \eta \ones{m}^T = \bar Y$,
so that the loss $\ell((\hat W_j, \hat b_j)_{j=1}^2) = \half \nfro{\bar Y - Y}^2 = \ell_0(\bar W)$. 

So far, we checked that $(\hat W_j, \hat b_j)_{j=1}^2$ has the same empirical risk as a linear least squares solution. It now remains to show that this point is indeed a local minimum of $\ell$. 
To that end, we consider the perturbed parameters $(\hat W_j+\Delta_j, \hat b_j+\delta_j)_{j=1}^2$, and check their risk is always larger. 
A useful point is that since $\bar W$ is a minimum of $\ell_0(R) = \half \nfro{R \tilde X - Y}^2$, we have
\begin{equation}
\label{eq:thm2identity}
(\bar W \tilde X - Y) \tilde X^T = (\bar Y - Y) \begin{bmatrix} X^T & \ones{m} \end{bmatrix}= 0,
\end{equation}
so $(\bar Y - Y) X^T = 0$ and $(\bar Y - Y) \ones{m} = 0$. 
For small enough perturbations,
$(\hat W_1 + \Delta_1) x_i + (\hat b_1 + \delta_1) > 0$ still holds for all $i$.
So, we can observe that
\begin{align*}
\ell((\hat W_j+\Delta_j, \hat b_j+\delta_j)_{j=1}^2)
\!= \half \nfro{\bar Y - Y 
	+ \tilde \Delta X
	+ \tilde \delta \ones{m}^T}^2
\!= \half \nfro{\bar Y - Y}^2 + \half \nfro {\tilde \Delta X + \tilde \delta \ones{m}^T}^2,\numberthis \label{eq:thm2step1}
\end{align*}
where $\tilde \Delta$ and $\tilde \delta$ are
$\tilde \Delta \defeq s_+ (\hat W_2  \Delta_1 + \Delta_2 \hat W_1 + \Delta_2 \Delta_1)$ and 
$\tilde \delta \defeq s_+ (\hat W_2 \delta_1 + \Delta_2 \hat b_1 + \Delta_2 \delta_1) + \delta_2$;
they are aggregated perturbation terms.
We used \eqref{eq:thm2identity} to obtain the last equality of \eqref{eq:thm2step1}.
Thus, 
$\ell((\hat W_j+\Delta_j, \hat b_j+\delta_j)_{j=1}^2) \geq \ell( (\hat W_j, \hat b_j)_{j=1}^2 )$ for small perturbations,
proving $(\hat W_j, \hat b_j)_{j=1}^2$ is indeed a local minimum of $\ell$. Since this is true for arbitrary $\alpha > 0$,
there are infinitely many such local minima. We can also construct similar local minima by permuting hidden nodes, etc.

\paragraph{Step 2: A point strictly better than the local minimum.}
The proof of this step is more involved. In the previous step, we ``pushed'' all the input to the hidden nodes to positive side, and took advantage of ``local linearity'' of the hidden nodes near $(\hat W_j, \hat b_j)_{j=1}^2$. But to construct parameters $(\tilde W_j, \tilde b_j)_{j=1}^2$ that have strictly smaller risk than $(\hat W_j, \hat b_j)_{j=1}^2$ (to prove that $(\hat W_j, \hat b_j)_{j=1}^2$ is a spurious local minimum), we make the sign of inputs to the hidden nodes different depending on data.

To this end, we sort the indices of data points in increasing order of $\bar y_i$; i.e., $\bar y_1 \leq \bar y_2 \leq \cdots \leq \bar y_m$. Define the set
$\bdindexset \defeq  \{ j \in [m-1] \mid \sum\nolimits_{i \leq j} (\bar y_i - y_i) \neq 0, \bar y_j < \bar y_{j+1} \}$.
The remaining construction is divided into two cases: $\bdindexset \neq \emptyset$ and $\bdindexset = \emptyset$, whose main ideas are essentially the same. We present the proof for $\bdindexset \neq \emptyset$, and defer the other case to Appendix~\ref{sec:thm2s2c2} as it is rarer, and its proof, while instructive for its perturbation argument, is technically too involved. 

\textbf{Case 1: $\bdindexset \neq \emptyset$.}
Pick any $j_0 \in \bdindexset$.
We can observe that $\sum\nolimits_{i \leq j_0} (\bar y_i - y_i) = - \sum\nolimits_{i > j_0} (\bar y_i - y_i)$, because of \eqref{eq:thm2identity}.
Define $\beta = \frac{\bar y_{j_0} + \bar y_{j_0+1}}{2}$, so that
$\bar y_i - \beta < 0$ for all $i \leq j_0$ and $\bar y_i - \beta > 0$ for all $i > j_0$.
Then, let $\gamma$ be a constant satisfying $0< |\gamma| \leq \frac{\bar y_{j_0+1} - \bar y_{j_0}}{4}$,
whose value will be specified later.
Since $|\gamma|$ is small enough, $\sign(\bar y_i-\beta) = \sign(\bar y_i-\beta + \gamma) =  \sign(\bar y_i-\beta - \gamma)$.
Now select parameters
\begin{align*}
\tilde W_1 = 
\begin{bmatrix}
\veccomp{\bar W}{[d_x]} \\
-\veccomp{\bar W}{[d_x]} \\
\zeros{(d_1-2) \times d_x}
\end{bmatrix}
,~
\tilde b_1 = 
\begin{bmatrix}
\veccomp{\bar W}{d_x+1}-\beta+\gamma\\
-\veccomp{\bar W}{d_x+1}+\beta+\gamma\\
\zeros{d_1-2}
\end{bmatrix}
,~
\tilde W_2 = \tfrac{1}{s_++s_-}  \begin{bmatrix} 1 & -1 & \zeros{d_1-2}^T \end{bmatrix}
, ~
\tilde b_2 = \beta.
\end{align*}
Recall again that $\veccomp{\bar W}{[d_x]} x_i + \veccomp{\bar W}{d_x+1} = \bar y_i$.
For $i \leq j_0$, $\bar y_i - \beta + \gamma < 0$ and $- \bar y_i + \beta + \gamma > 0$, so 
\begin{align*}
\hat y_i 
= \frac{s_-  ( \bar y_i  - \beta + \gamma) }{ s_++s_-}
- \frac{s_+ (- \bar y_i + \beta +\gamma )}{ s_++s_-} 
+ \beta
= \bar y_i - \frac{s_+ - s_-}{s_+ + s_-} \gamma.
\end{align*}
Similarly, for $i > j_0$, $\bar y_i - \beta + \gamma > 0$ and $- \bar y_i + \beta + \gamma < 0$ results in $\hat y_i = \bar y_i + \frac{s_+ - s_-}{s_+ + s_-} \gamma$.
Here, we push the outputs $\hat y_i$ of the network by $\frac{s_+ - s_-}{s_+ + s_-} \gamma$ from $\bar y_i$,
and the direction of the ``push'' varies depending on whether $i \leq j_0$ or $i > j_0$. 

The empirical risk for this choice of parameters is
\begin{align*}
\ell((\tilde W_j, \tilde b_j)_{j=1}^2) &= \bhalf \sum\nolimits_{i \leq j_0} \Big (\bar y_i - \frac{s_+ - s_-}{s_+ + s_-} \gamma - y_i \Big )^2
+ \bhalf \sum\nolimits_{i > j_0} \Big (\bar y_i + \frac{s_+ - s_-}{s_+ + s_-} \gamma - y_i \Big )^2\\
&= \ell_0(\bar W) -  2 \left [ \sum\nolimits_{i \leq j_0} (\bar y_i - y_i)  \right ] \frac{s_+ - s_-}{s_+ + s_-} \gamma + O(\gamma^2).
\end{align*}
Since $\sum_{i \leq j_0} (\bar y_i - y_i) \neq 0$ and $s_+ \neq s_-$, 
we can choose $\sign(\gamma) = \sign([\sum_{i \leq j_0} (\bar y_i - y_i) ](s_+ - s_-))$, and choose small $|\gamma|$ so that $\ell((\tilde W_j, \tilde b_j)_{j=1}^2) < \ell_0(\bar W) = \ell((\hat W_j, \hat b_j)_{j=1}^2)$,
proving that $(\hat W_j, \hat b_j)_{j=1}^2$ is a spurious local minimum.

\vspace*{-5pt}
\section{Counterexample: bad local minima for many activations}
\vspace*{-5pt}
\label{sec:othernonlin}
The proof of Theorem~\ref{thm:piecelin} crucially exploits the piecewise linearity of the activation functions. Thus, one may wonder whether the spurious local minima seen there are an artifact of the specific nonlinearity.  We show below that this is \emph{not} the case. We provide a counterexample nonlinear network and a dataset for which a wide range of nonlinear activations result in a local minimum that is strictly inferior to the global minimum with exactly zero empirical risk. Examples of such activation functions include popular activation functions such as sigmoid, tanh, arctan, ELU, SELU, and ReLU.

We consider again the squared error empirical risk of a one-hidden-layer nonlinear neural network: 
\begin{align*}
\ell( (W_j,b_j)_{j=1}^2 )
:= \tfrac12 \nfro{W_2 h (W_1 X \! + \! b_1 \ones{m}^T) \!+\!  b_2 \ones{m}^T \!-\! Y}^2,
\end{align*}
where we fix $d_x = d_1 = 2$ and $d_y = 1$.
Also, let $h^{(k)}(x)$ be the $k$-th derivative of $h : \reals \mapsto \reals$,
whenever it exists at $x$. For short, let $h'$ and $h''$ denote the first and second derivatives.

\vspace*{-4pt}
\subsection{Main results and discussion}
\begin{theorem}
	\label{thm:othernonlin}
	Let the loss $\ell( (W_j,b_j)_{j=1}^2)$ and network be as defined above. Consider the dataset
        \begin{small}
	\begin{equation*}
	X = \begin{bmatrix} 1 & 0 & \half \\ 0 & 1 & \half \end{bmatrix},~
	Y = \begin{bmatrix} 0 & 0 & 1 \end{bmatrix}.
      \end{equation*}
    \end{small}%
    For this network and dataset the following results hold:
	\begin{enumerate}[leftmargin=1em]
          \setlength{\itemsep}{0pt}
		\item If there exist real numbers $v_1, v_2, v_3, v_4 \in \reals$ such that 
                  \begin{enumerate}[label={\small(C\ref{thm:othernonlin}.\arabic*)}, leftmargin=25pt]
			\item $h(v_1) h(v_4) = h(v_2) h(v_3)$, and
			\item $h(v_1) h \left ( \frac{v_3+v_4}{2} \right ) \neq h(v_3) h \left (\frac{v_1+v_2}{2} \right )$,
		\end{enumerate}
		then there is a tuple $(\tilde W_j, \tilde b_j)_{j=1}^2$ at which $\ell$ equals $0$. 
		\item If there exist real numbers $v_1, v_2, u_1, u_2 \in \reals$ such that the following conditions hold:
		\begin{enumerate}[label={\small(C\ref{thm:othernonlin}.\arabic*)}, leftmargin=25pt]
			\setlength{\itemsep}{0pt}
			\setcounter{enumii}{2}
			\item $u_1 h(v_1) + u_2 h(v_2) = \frac{1}{3}$,
			\item $h$ is infinitely differentiable at $v_1$ and $v_2$,
			\item there exists a constant $c > 0$ such that $|h^{(n)}(v_1)| \leq c^n n!$ and $|h^{(n)}(v_2)| \leq c^n n!$.
			\item $(u_1 h'(v_1))^2 + \frac{u_1 h''(v_1)}{3} > 0$,
			\item $(u_1 h'(v_1) u_2 h'(v_2))^2 \!<\! ( (u_1 h'(v_1))^2\! + \!\frac{u_1 h''(v_1)}{3} ) ( (u_2 h'(v_2))^2 + \frac{u_2 h''(v_2)}{3} )$,
		\end{enumerate}
		then there exists a tuple $(\hat W_j, \hat b_j)_{j=1}^2$ such that the output of the network is the same as the linear least squares model, the risk $\ell( (\hat W_j, \hat b_j)_{j=1}^2 ) = \frac{1}{3}$, and $(\hat W_j, \hat b_j)_{j=1}^2$ is a local minimum of $\ell$.
	\end{enumerate}
\end{theorem}
Theorem~\ref{thm:othernonlin} shows that for this architecture and dataset, activations that satisfy (C\ref{thm:othernonlin}.1)--(C\ref{thm:othernonlin}.7) introduce at least one spurious local minimum. 
Notice that the empirical risk is zero at the global minimum. This means that the data $X$ and $Y$ can actually be ``generated'' by the network, which satisfies the realizability assumption that others use~\citep{safran2017spurious, du2017gradient, wu2018no}.
Notice that our counterexample is ``easy to fit,'' and yet, there exists a local minimum that is not global. This leads us to conjecture that with harder datasets, the problems with spurious local minima could be worse.
The proof of Theorem~\ref{thm:othernonlin} can be found in Appendix~\ref{sec:thm3}.

\textbf{Discussion.} Note that the conditions~(C\ref{thm:othernonlin}.1)--(C\ref{thm:othernonlin}.7) only require \emph{existence} of certain real numbers rather than some \emph{global} properties of activation $h$, hence are not as restrictive as they look. Conditions~(C\ref{thm:othernonlin}.1)--(C\ref{thm:othernonlin}.2) come from a choice of tuple $(\tilde W_j, \tilde b_j)_{j=1}^2$ that perfectly fits the data. Condition~(C\ref{thm:othernonlin}.3) is necessary for constructing $(\hat W_j, \hat b_j)_{j=1}^2$ with the same output as the linear least squares model, and Conditions~(C\ref{thm:othernonlin}.4)--(C\ref{thm:othernonlin}.7) are needed for showing local minimality of $(\hat W_j, \hat b_j)_{j=1}^2$ via Taylor expansions.
The class of functions that satisfy conditions~(C\ref{thm:othernonlin}.1)--(C\ref{thm:othernonlin}.7) is quite large,
and includes the nonlinear activation functions used in practice. The next corollary highlights this observation (for a proof with explicit choices of the involved real numbers, please see Appendix~\ref{sec:cor2}).
\begin{corollary}
	\label{cor:actftnex}
	For the counterexample in Theorem~\ref{thm:othernonlin}, the set of activation functions satisfying conditions (C\ref{thm:othernonlin}.1)--(C\ref{thm:othernonlin}.7) include sigmoid, tanh, arctan, quadratic, ELU, and SELU.
\end{corollary}

Admittedly, Theorem~\ref{thm:othernonlin} and Corollary~\ref{cor:actftnex} give one counterexample instead of stating a claim about generic datasets. Nevertheless, this example shows that for many practical nonlinear activations, the desirable ``local minimum is global'' property cannot hold even for realizable datasets, suggesting that the situation could be worse for non-realizable ones.

\textbf{Remark: ``ReLU-like'' activation functions.}
Recall the piecewise linear nonnegative homogeneous activation function $\relulikeactfun$.
They do not satisfy condition (C\ref{thm:othernonlin}.7), so Theorem~\ref{thm:othernonlin} cannot be directly applied. Also, if $s_- = 0$ (i.e., ReLU), conditions (C\ref{thm:othernonlin}.1)--(C\ref{thm:othernonlin}.2) are also violated.
However, the statements of Theorem~\ref{thm:othernonlin} hold even for $\relulikeactfun$, which is shown in Appendix~\ref{sec:thm3forRelus}.
Recalling again $s_+ = 1+\epsilon$ and $s_- = 1$, this means that even with the ``slightest'' nonlinearity in activation function,
the network has a global minimum with risk zero while there exists a bad local minimum that performs just as linear least squares models.
In other words, ``local minima are global'' property is rather brittle and can only hold for linear neural networks.
Another thing to note is that in Appendix~\ref{sec:thm3forRelus}, the bias parameters are all zero, for both $(\tilde W_j, \tilde b_j)_{j=1}^2$ and $(\hat W_j, \hat b_j)_{j=1}^2$. For models without bias parameters, $(\hat W_j)_{j=1}^2$ is still a spurious local minimum, thus showing that \citet{wu2018no} fails to extend to empirical risks and non-unit weight vectors.



%% file: linear.tex
\vspace*{-5pt}
\section{Global optimality in linear networks}
\vspace*{-4pt}
\label{sec:linear}
In this section we present our results on deep linear neural networks. Assuming that the hidden layers are at least as wide as either the input or output, we show that critical points of the loss with a multilinear parameterization inherit the type of critical points of the loss with a linear parameterization. As a corollary, we show that for differentiable losses whose critical points are globally optimal, deep linear networks have \emph{only global minima or saddle points}. Furthermore, we provide an efficiently checkable condition for global minimality.

Suppose the network has $H$ hidden layers having widths $d_1, \dots, d_H$. To ease notation, we set $d_0 = d_x$ and $d_{H+1} = d_y$. The weights between adjacent layers are kept in matrices $W_j \in \reals^{d_j \times d_{j-1}}$ ($j \in [H+1]$), and the output $\hat Y$ of the network is given by the product of weight matrices with the data matrix: $\hat Y = W_{H+1}W_{H}\cdots W_{1} X$.
Let $(W_j)_{j=1}^{H+1}$ be the tuple of all weight matrices, and $W_{i:j}$ denote the product $W_i W_{i-1} \cdots W_{j+1} W_j$ for $i \geq j$, and the identity for $i = j-1$. We consider the empirical risk $\ell((W_j)_{j=1}^{H+1})$, which, for linear networks assumes the form
\begin{equation}
\label{eq:linlossdef}
\ell( (W_j)_{j=1}^{H+1} ) \defeq \ell_0(W_{H+1:1}),
\end{equation}
where $\ell_0$ is a suitable differentiable loss. 
For example, when $\ell_0(R) = \half \nfro{R X - Y}^2$, 
$\ell( (W_j)_{j=1}^{H+1} ) = \tfrac12\nfro{W_{H+1:1} X - Y}^2 = \ell_0(W_{H+1:1})$.
Lastly, we write $\gradlzero(M) \equiv \nabla_R \ell_0(R) \vert_{R = M}$.

\textbf{Remark: bias terms.} We omit the bias terms $b_1, \dots, b_{H+1}$ here. This choice is for simplicity; models with bias can be handled by the usual trick of augmenting data and weight matrices. 

\vspace*{-4pt}
\subsection{Main results and discussion}
\vspace*{-4pt}
We are now ready to state our first main theorem, whose proof is deferred to Appendix~\ref{sec:thm1}.
\begin{theorem}
	\label{thm:linear}
	Suppose that for all $j$, 
	$d_j \geq \min \{d_x, d_y\}$, and that the loss $\ell$ is given by~\eqref{eq:linlossdef}, where $\ell_0$ is differentiable on $\reals^{d_y \times d_x}$. 
	For any critical point $(\hat W_j)_{j=1}^{H+1}$ of the loss $\ell$, the following claims hold:
	\begin{enumerate}
		\setlength{\itemsep}{0pt}
		\item If $\,\gradlzero(\hat W_{H+1:1}) \neq 0$, then $(\hat W_j)_{j=1}^{H+1}$ is a saddle of $\ell$. 
		\item If $\gradlzero(\hat W_{H+1:1}) = 0$, then
                  \begin{enumerate}
                    \setlength{\itemsep}{1pt}
			\item $(\hat W_j)_{j=1}^{H+1}$ is a local min (max) of $\ell$ if $\hat W_{H+1:1}$ is a local min (max) of $\ell_0$; moreover, 
			\item $(\hat W_j)_{j=1}^{H+1}$ is a global min (max) of $\ell$ if and only if $\hat W_{H+1:1}$ is a global min (max) of $\ell_0$.
		\end{enumerate}
		\item If there exists $j^* \in [H+1]$ such that $\hat W_{H+1:j^*+1}$ has full row rank and $\hat W_{j^*-1:1}$ has full column rank, then $\gradlzero(\hat W_{H+1:1}) = 0$, so 2(a) and 2(b) hold. Also,
                  \begin{enumerate}
                    \vspace*{-4pt}
                    \setlength{\itemsep}{1pt}
			\item $\hat W_{H+1:1}$ is a local min (max) of $\ell_0$ if $(\hat W_j)_{j=1}^{H+1}$ is a local min (max) of $\ell$.
		\end{enumerate}
	\end{enumerate}
\end{theorem}
Let us paraphrase Theorem~\ref{thm:linear} in words. In particular, it states that if the hidden layers are ``wide enough'' so that the product $W_{H+1:1}$ can attain full rank and if the loss $\ell$ assumes the form~\eqref{eq:linlossdef} for a differentiable loss $\ell_0$, then the type (optimal or saddle point) of a critical point $(\hat W_j)_{j=1}^{H+1}$ of $\ell$ is governed by the behavior of $\ell_0$ at the product $\hat W_{H+1:1}$.

Note that for any critical point $(\hat W_j)_{j=1}^{H+1}$ of the loss $\ell$, either $\gradlzero(\hat W_{H+1:1}) \neq 0$ or $\gradlzero(\hat W_{H+1:1}) = 0$. Parts~1 and 2 handle these two cases. Also observe that the condition in Part~3 implies $\nabla\ell_0=0$, so Part~3 is a refinement of Part~2.
A notable fact is that a sufficient condition for Part~3 is $\hat W_{H+1:1}$ having full rank. 
For example, if $d_x \geq d_y$, full-rank $\hat W_{H+1:1}$ implies $\rank(\hat W_{H+1:2}) = d_y$, whereby the condition in Part~3 holds with $j^* = 1$.

If $\hat W_{H+1:1}$ is not critical for $\ell_0$, then $(\hat W_j)_{j=1}^{H+1}$ must be a saddle point of $\ell$. If $\hat W_{H+1:1}$ is a local min/max of $\ell_0$, $(\hat W_j)_{j=1}^{H+1}$ is also a local min/max of $\ell$. Notice, however, that Part~2(a) does not address the case of saddle points; when $\hat W_{H+1:1}$ is a saddle point of $\ell_0$, the tuple $(\hat W_j)_{j=1}^{H+1}$ can behave arbitrarily. However, with the condition in Part~3, statements~2(a) and 3(a) hold at the same time, so that $\hat W_{H+1:1}$ is a local min/max of $\ell_0$ \emph{if and only if} $(\hat W_j)_{j=1}^{H+1}$ is a local min/max of $\ell$. Observe that the same ``if and only if'' statement holds for saddle points due to their definition; in summary, the types (min/max/saddle) of the critical points $(\hat W_j)_{j=1}^{H+1}$ and $\hat W_{H+1:1}$ match exactly.

Although Theorem~\ref{thm:linear} itself is of interest, the following corollary highlights its key implication for deep linear networks.
\begin{corollary}
	\label{cor:linearcvx}
	In addition to the assumptions in Theorem~\ref{thm:linear}, assume that any critical point of $\ell_0$ is a global min (max). For any critical point $(\hat W_j)_{j=1}^{H+1}$ of $\ell$, 
	if $\gradlzero(\hat W_{H+1:1}) \neq 0$, then $(\hat W_j)_{j=1}^{H+1}$ is a saddle of $\ell$, while if $\gradlzero(\hat W_{H+1:1}) = 0$, then $(\hat W_j)_{j=1}^{H+1}$ is a global min (max) of $\ell$.
\end{corollary}
\vspace*{-4pt}
\begin{proof}
	If $\gradlzero(\hat W_{H+1:1}) \neq 0$, then $\hat W_{H+1:1}$ is a saddle point by Theorem~\ref{thm:linear}.1. If $\gradlzero(\hat W_{H+1:1}) = 0$, then $\hat W_{H+1:1}$ is a global min (max) of $\ell_0$ by assumption.
	By Theorem~\ref{thm:linear}.2(b), $(\hat W_j)_{j=1}^{H+1}$ must be a global min (max) of $\ell$.
	\vspace*{-4pt}
\end{proof}
Corollary~\ref{cor:linearcvx} shows that for any differentiable loss function $\ell_0$ whose critical points are global minima, the loss $\ell$ has only global minima and saddle points, therefore satisfying the ``local minima are global'' property. In other words, for such an $\ell_0$, the multilinear re-parametrization introduced by deep linear networks \emph{does not introduce any spurious local minima/maxima};  it only introduces saddle points. 
Importantly, Corollary~\ref{cor:linearcvx} also provides a checkable condition that distinguishes global minima from saddle points. Since $\ell$ is nonconvex, it is remarkable that such a simple necessary and sufficient condition for global optimality is available.

Our result generalizes previous works on linear networks such as \citet{kawaguchi2016deep, yun2018global, zhou2018critical}, because it provides conditions for global optimality for a broader range of loss functions without assumptions on datasets. \citet{laurent2018deep} proved that if $(\hat W_j)_{j=1}^{H+1}$ is a local min of $\ell$, then $\hat W_{H+1:1}$ is a critical point of $\ell_0$. First, observe that this result is implied by Theorem~\ref{thm:linear}.1. So our result, which was proved in parallel and independently, is strictly more general.
With additional assumption that critical points of $\ell_0$ are global minima, \citet{laurent2018deep} showed that ``local min is global'' property holds for linear neural networks; 
our Corollay~\ref{cor:linearcvx} gives a simple and efficient test condition as well as proving there are only global minima and saddles, which is clearly stronger.


%% file: apxlin.tex
\section{Notation}
We first list notation used throughout the appendix.
For integers $a \le b$, $[a,b]$ denotes the set of integers between them. We write $[b]$, if $a = 1$. For a vector $v$, we use $\veccomp{v}{i}$ to denote its $i$-th component, while  $\veccomp{v}{[i]}$ denotes a vector comprised of the first $i$ components of $v$. Let $\ones{d}$ (or $\zeros{d}$) be the all ones (zeros) column vector in $\reals^d$. For a subspace $V\subseteq \reals^d$, we denote by $V^\perp$ its orthogonal complement.

For a matrix $A$, $\matent{A}{i}{j}$ is the $(i,j)$-th entry and $\matcol{A}{j}$ its $j$-th column. Let $\sigmax(A)$ and $\sigmin(A)$ denote the largest and smallest singular values of $A$, respectively; $\rowsp(A)$, $\colsp(A)$, $\rank(A)$, and $\nfro {A}$ denote respectively the row space, column space, rank, and Frobenius norm of matrix $A$. Let $\nulsp(A) \defeq \{ v \mid Av = 0 \}$ and $\lnulsp(A) \defeq  \{ v \mid v^T A = 0 \}$ be the null space and the left-null space of $A$, respectively.
When $A$ is a square matrix, let $\tr(A)$ be the trace of $A$.
For matrices $A$ and $B$ of the same size, $\dotprod{A}{B} = \tr(A^TB)$ denotes the usual trace inner product of $A$ and $B$.
Equivalently, $\dotprod{A}{B} = \tr(A^T B) = \tr(A B^T)$.
Let $\zeros{d \times m}$ be the all zeros matrix in $\reals^{d \times m}$. 

\section{Proof of Theorem~\ref{thm:piecelin}, Step 2, Case 2}
\label{sec:thm2s2c2}
\paragraph{Case 2. $\bdindexset = \emptyset$.}
We start with a lemma discussing what $\bdindexset = \emptyset$ implies.
\begin{lemma}
	\label{lem:setiempty}
	If $\bdindexset = \emptyset$, the following statements hold:
	\begin{enumerate}
		\item There are some $\bar y_j$'s that are duplicate; i.e.\ for some $i\neq j$, $\bar y_{i} = \bar y_{j}$.
		\item If $\bar y_j$ is non-duplicate, meaning that $\bar y_{j-1} < \bar y_j < \bar y_{j+1}$, $\bar y_j = y_j$ holds.
		\item If $\bar y_j$ is duplicate, $\sum_{i : \bar y_i = \bar y_j} (\bar y_i - y_i) = 0$ holds.
		\item There exists at least one duplicate $\bar y_j$ such that, for that $\bar y_j$, there exist at least two different $i$'s that satisfy $\bar y_i = \bar y_j$ and $\bar y_i \neq y_i$.
	\end{enumerate}
\end{lemma}
\begin{proof}
	We prove this by showing if any of these statements are not true, then we have $\bdindexset \neq \emptyset$ or a contradiction.
	\begin{enumerate}
		\item If all the $\bar y_j$'s are distinct and $\bdindexset = \emptyset$, by definition of $\bdindexset$, $\bar y_j = y_j$ for all $j$. 
		This violates our assumption that linear models cannot perfectly fit $Y$.
		\item If we have $\bar y_j \neq y_j$ for a non-duplicate $\bar y_j$, at least one of the following statements must hold:
		$\sum_{i\leq j-1} (\bar y_i - y_i) \neq 0$ or $\sum_{i\leq j} (\bar y_i - y_i) \neq 0$, meaning that $j-1 \in \bdindexset$ or $j \in \bdindexset$.
		\item Suppose $\bar y_j$ is duplicate and $\sum_{i : \bar y_i = \bar y_j} (\bar y_i - y_i) \neq 0$.
		Let $k = \min \{ i \mid \bar y_i = \bar y_j \}$ and $l = \max \{ i \mid \bar y_i = \bar y_j \}$. 
		Then at least one of the following statements must hold:
		$\sum_{i \leq k-1} (\bar y_i - y_i) \neq 0$ or $\sum_{i \leq l} (\bar y_i - y_i) \neq 0$.
		If $\sum_{i \leq k-1} (\bar y_i - y_i) \neq 0$, we can also see that $\bar y_{k-1} < \bar y_k$, so $k-1 \in \bdindexset$.
		Similarly, if $\sum_{i \leq l} (\bar y_i - y_i) \neq 0$, then $l \in \bdindexset$.
		\item Since $\sum_{i : \bar y_i = \bar y_j} (\bar y_i - y_i) = 0$ holds for any duplicate $\bar y_j$, if $\bar y_i \neq y_i$ holds for one $i$ then there must be at least two of them that satisfies $\bar y_i \neq y_i$. If this doesn't hold for all duplicate $\bar y_i$, with Part 2 this means that $\bar y_j = y_j$ holds for all $j$. This violates our assumption that linear models cannot perfectly fit $Y$.
	\end{enumerate}
\end{proof}
From Lemma~\ref{lem:setiempty}.4, we saw that there is a duplicate value of $\bar y_j$ such that
some of the data points $i$ satisfy $\bar y_i = \bar y_j$ and $\bar y_i \neq y_i$.
The proof strategy in this case is essentially the same, 
but the difference is that we choose one of such duplicate $\bar y_j$,
and then choose a vector $v \in \reals^{d_x}$ to ``perturb'' 
the linear least squares solution $\veccomp{\bar W}{[d_x]}$
in order to break the tie between $i$'s that satisfies $\bar y_i = \bar y_j$ and $\bar y_i \neq y_i$.

We start by defining the minimum among such duplicate values $\bar y^*$ of $\bar y_j$'s, 
and a set of indices $j$ that satisfies $\bar y_j = \bar y^*$. 
\begin{align*}
&\bar y^* = \min \{\bar y_j \mid \exists i \neq j \text{ such that } \bar y_i = \bar y_j \text{ and } \bar y_i \neq y_i \},\\
&\bdindexset^* = \{ j \in [m] \mid \bar y_j = \bar y^* \}.
\end{align*}
Then, we define a subset of $\bdindexset^*$:
\begin{align*}
\bdindexset^*_{\neq} = \{ j \in \bdindexset^* \mid \bar y_j \neq y_j \}.
\end{align*}
By Lemma~\ref{lem:setiempty}.4, cardinality of $\bdindexset^*_{\neq}$ is at least two.
Then, we define a special index in $\bdindexset^*_{\neq}$:
\begin{align*}
j_1 = \argmax_{j \in \bdindexset^*_{\neq}} \ltwo{x_j}, ~
\end{align*}
Index $j_1$ is the index of the ``longest'' $x_j$ among elements in $\bdindexset^*_{\neq}$.
Using the definition of $j_1$, we can partition $\bdindexset^*$ into two sets:
\begin{align*}
\bdindexset^*_{\geq} = \{ j \in \bdindexset^* \mid \< x_j, x_{j_1} \> \geq \ltwo{x_{j_1}}^2 \},~
\bdindexset^*_{<} = \{ j \in \bdindexset^* \mid \< x_j, x_{j_1} \> < \ltwo{x_{j_1}}^2 \}.
\end{align*}
For the indices in $\bdindexset^*$, we can always switch the indices without loss of generality.
So we can assume that $j \leq j_1 = \max \bdindexset^*_{\geq}$ for all $j \in \bdindexset^*_{\geq}$  and $j > j_1$ for all $j \in \bdindexset^*_{<}$.

We now define a vector that will be used as the ``perturbation'' to $\veccomp{\bar W}{[d_x]}$.
Define a vector $v \in \reals^{d_x}$, which is a scaled version of $x_{j_1}$:
\begin{equation*}
v = \frac{g}{M \ltwo{x_{j_1}}} x_{j_1},
\end{equation*}
where the constants $g$ and $M$ are defined to be
\begin{equation*}
g = \frac{1}{4} \min \left \{ | \bar y_i - \bar y_j | \mid i,j \in [m], \bar y_i \neq \bar y_j \right \}, ~
M = \max_{i \in [m]} \ltwo{x_i}.
\end{equation*}
The constant $M$ is the largest $\ltwo{x_i}$ among all the indices, and $g$ is one fourth times the minimum gap between all distinct values of $\bar y_i$.

Now, consider perturbing $\veccomp{\bar W}{[d_x]}$ by a vector $-\alpha v^T$. where $\alpha \in (0,1]$ will be specified later.
Observe that
\begin{align*}
&\left (\bar W - \begin{bmatrix}\alpha v^T & 0\end{bmatrix} \right ) \begin{bmatrix} x_i \\ 1\end{bmatrix} = \bar W \begin{bmatrix} x_i \\ 1\end{bmatrix} - \alpha v^Tx_i = \bar y_i - \alpha v^T x_i.
\end{align*}
Recall that $j \leq j_1 = \max \bdindexset^*_{\geq}$ for all $j \in \bdindexset^*_{\geq}$ and $j > j_1$ for all $j \in \bdindexset^*_{<}$.
We are now ready to present the following lemma:
\begin{lemma}
	\label{lem:roleofbeta}
	Define 
	\begin{align*}
	j_2 = \argmax_{j \in \bdindexset^*_{<} } \< x_j , x_{j_1} \>,~
	\beta = \bar y^* - \frac{\alpha}{2} v^T(x_{j_1}+x_{j_2}).
	\end{align*}
	Then, 
	\begin{align*}
	&\bar y_{i} - \alpha v^T x_{i} - \beta < 0 \text{ for all } i \leq j_1, \\
	&\bar y_{i} - \alpha v^T x_{i} - \beta > 0 \text{ for all } i > j_1.
	\end{align*}
	Also, $\sum_{i > j_1} (\bar y_i - y_i) - \sum_{i \leq j_1} (\bar y_i - y_i) = -2(\bar y_{j_1} - y_{j_1}) \neq 0$.
\end{lemma}
\begin{proof}
	First observe that, for any $x_i$, $| \alpha v^T x_i | \leq \alpha \ltwo {v} \ltwo{ x_i} \leq \frac{g}{M} \ltwo{x_i} \leq g$.
	By definition of $g$, we have $2g < \bar y_j-\bar y_i$ for any $\bar y_i < \bar y_j$.
	Using this, we can see that
	\begin{equation}
	\label{eq:vgdef}
	\bar y_i  < \bar y_j \implies \bar y_i - \alpha v^T x_i \leq \bar y_i + g < \bar y_j - g \leq \bar y_j - \alpha v^T x_j .
	\end{equation}
	In words, if $\bar y_i$ and $\bar y_j$ are distinct and there is an order $\bar y_i < \bar y_j$, 
	perturbation of $\veccomp{\bar W}{[d_x]}$ by $-\alpha v^T$ does not change the order.
	Also, since $v$ is only a scaled version of $x_{j_1}$, from the definitions of $\bdindexset^*_{\geq}$ and $\bdindexset^*_{<}$,
	\begin{equation}
	\label{eq:j1def}
	v^T (x_{j} - x_{j_1}) \geq 0 \text{ for } j \in \bdindexset^*_{\geq} \text{ and }
	v^T (x_{j} - x_{j_1}) < 0 \text{ for } j \in \bdindexset^*_{<}. 
	\end{equation}
	By definition of $j_2$, 
	\begin{equation}
	\label{eq:j2def}
	v^T (x_{j_2} - x_{j_1}) < 0 \text{ and } 
	v^T(x_{j_2} - x_j) \geq 0 \text{ for all } j \in \bdindexset^*_{<}.
	\end{equation}
	
	It is left to prove the statement of the lemma using case analysis, using the inequalities \eqref{eq:vgdef}, \eqref{eq:j1def}, and \eqref{eq:j2def}.
	For all $i$'s such that $\bar y_i < \bar y^* = \bar y_{j_1}$,
	\begin{align*}
	\bar y_i - \alpha v^T x_i - \beta
	&=  \bar y_i - \alpha v^T x_i - \bar y^* + \frac{\alpha}{2} v^T(x_{j_1}+x_{j_2})\\
	&= (\bar y_i - \alpha v^T x_i) - (\bar y^* - \alpha v^T x_{j_1} ) + \frac{\alpha}{2} v^T (x_{j_2}-x_{j_1}) < 0.
	\end{align*}
	Similarly, for all $i$ such that $\bar y_i > \bar y^* = \bar y_{j_2}$,
	\begin{equation*}
	\bar y_{i} - \alpha v^T x_i - \beta
	= (\bar y_i - \alpha v^T x_i) - (\bar y^* - \alpha v^T x_{j_2} ) + \frac{\alpha}{2} v^T (x_{j_1}-x_{j_2}) > 0.
	\end{equation*}
	For $j \in \bdindexset^*_{\geq}$ ($j \leq j_1$), we know $\bar y_{j} = \bar y^*$, so
	\begin{align*}
	\bar y_{j} - \alpha v^T x_j - \beta
	&= \left (\bar y^* - \alpha v^T x_j \right  ) - \left (\bar y^* -\frac{\alpha}{2} v^T (x_{j_1} + x_{j_2})\right )\\
	&= \alpha v^T [(x_{j_1} - x_{j})] + \frac{\alpha}{2} v^T [(x_{j_2} - x_{j_1})] < 0.
	\end{align*}
	Also, for $j \in \bdindexset^*_{<}$ ($j > j_1$),
	\begin{align*}
	\bar y_{j} - \alpha v^T x_j - \beta
	&= \left (\bar y^* - \alpha v^T x_j \right  ) - \left (\bar y^* -\frac{\alpha}{2} v^T (x_{j_1} + x_{j_2})\right )\\
	&= \frac{\alpha}{2} v^T[(x_{j_1} - x_{j})+(x_{j_2} - x_{j})] > 0.
	\end{align*}
	This finishes the case analysis and proves the first statements of the lemma.
	
	One last thing to prove is that $\sum_{i > j_1} (\bar y_i - y_i) - \sum_{i \leq j_1} (\bar y_i - y_i) = -2(\bar y_{j_1} - y_{j_1}) \neq 0$.
	Recall from Lemma~\ref{lem:setiempty}.2 that for non-duplicate $\bar y_j$, we have $\bar y_j = y_j$.
	Also by Lemma~\ref{lem:setiempty}.3 if $\bar y_j$ is duplicate, $\sum_{i:\bar y_i = \bar y_j} (\bar y_i - y_i) = 0$.
	So,
	\begin{equation*}
	\sum_{i > j_1} \left (\bar y_i - y_i \right ) - \sum_{i \leq j_1} \left (\bar y_i - y_i \right ) 
	= \sum_{i \in \bdindexset^*_{<}} \left (\bar y_i - y_i \right ) - \sum_{i \in \bdindexset^*_{\geq}} \left (\bar y_i - y_i \right ).
	\end{equation*}
	Recall the definition of $\bdindexset^*_{\neq} = \{ j \in \bdindexset^* \mid \bar y_j \neq y_j \}$.
	For $j \in \bdindexset^* \backslash \bdindexset^*_{\neq}$,
	$\bar y_j = y_j$. So,
	\begin{equation*}
	\sum_{i \in \bdindexset^*_{<}} \left (\bar y_i - y_i \right ) - \sum_{i \in \bdindexset^*_{\geq}} \left (\bar y_i - y_i \right )
	=
	\sum_{i \in \bdindexset^*_{<} \cap \bdindexset^*_{\neq}} \left (\bar y_i - y_i \right ) - \sum_{i \in \bdindexset^*_{\geq} \cap \bdindexset^*_{\neq}} \left (\bar y_i - y_i \right ).
	\end{equation*}
	
	Recall the definition of $j_1 = \argmax_{j \in \bdindexset^*_{\neq}} \ltwo{x_j}$.
	For any other $j \in {\bdindexset^*_{\neq} \backslash \{j_1\}}$,
	\begin{equation*}
	\ltwo{ x_{j_1} }^2 \geq \ltwo{ x_{j} } \ltwo{ x_{j_1} } \geq \< x_{j}, x_{j_1} \>,
	\end{equation*}
	where the first $\geq$ sign is due to definition of $j_1$, and the second is from Cauchy-Schwarz inequality.
	Since $x_{j_1}$ and $x_{j}$ are distinct by assumption, 
	they must differ in either length or direction, or both.
	So, we can check that at least one of ``$\geq$'' must be strict inequality, so 
	$\ltwo{ x_{j_1} }^2 > \< x_{j}, x_{j_1} \>$ for all $j \in {\bdindexset^*_{\neq} \backslash \{j_1\}}$.
	Thus,
	\begin{equation*}
	\bdindexset^*_{\neq} \backslash \{j_1\} = \bdindexset^*_{<} \cap \bdindexset^*_{\neq}
	\text{ and }
	\{j_1\} = \bdindexset^*_{\geq} \cap \bdindexset^*_{\neq},
	\end{equation*}
	proving that 
	\begin{equation*}
	\sum_{i > j_1} \left (\bar y_i - y_i \right ) - \sum_{i \leq j_1} \left (\bar y_i - y_i \right ) 
	= \sum_{j \in \bdindexset^*_{\neq} \backslash \{j_1\}} \left (\bar y_i - y_i \right ) - \left (\bar y_{j_1} - y_{j_1} \right ).
	\end{equation*}
	Also, by Lemma~\ref{lem:setiempty}.3,
	\begin{equation*}
	0 = \sum_{i \in \bdindexset^*} (\bar y_i - y_i) = \sum_{i \in \bdindexset^*_{\neq}} (\bar y_i - y_i)
	= (\bar y_{j_1} - y_{j_1}) + \sum_{j \in \bdindexset^*_{\neq} \backslash \{j_1\}} (\bar y_i - y_i).
	\end{equation*}
	Wrapping up all the equalities, we can conclude that
	\begin{equation*}
	\sum_{i > j_1} \left (\bar y_i - y_i \right ) - \sum_{i \leq j_1} \left (\bar y_i - y_i \right ) 
	= - 2\left (\bar y_{j_1} - y_{j_1} \right ),
	\end{equation*}
	finishing the proof of the last statement.
\end{proof}

It is time to present the parameters $(\tilde W_j, \tilde b_j)_{j=1}^2$, whose empirical risk is strictly smaller than the local minimum $(\hat W_j, \hat b_j)_{j=1}^2$ with a sufficiently small choice of $\alpha \in (0,1]$.
Now, let $\gamma$ be a constant such that 
\begin{equation}
\label{eq:c2gammadef}
\gamma = \sign((\bar y_{j_1} - y_{j_1})(s_+ - s_-)) \frac{\alpha v^T(x_{j_1} - x_{j_2})}{4}.
\end{equation}
Its absolute value is proportional to $\alpha \in (0, 1]$, which is a undetermined number that will be specified at the end of the proof.
Since $|\gamma|$ is small enough, we can check that
\begin{equation*}
\sign(\bar y_i - \alpha v^T x_i - \beta) = \sign(\bar y_i - \alpha v^T x_i - \beta + \gamma) = \sign(\bar y_i - \alpha v^T x_i - \beta - \gamma).
\end{equation*}
Then, assign parameter values
\begin{align*}
&\tilde W_1 = 
\begin{bmatrix}
\veccomp{\bar W}{[d_x]} - \alpha v^T\\
-\veccomp{\bar W}{[d_x]} + \alpha v^T\\
\zeros{(d_1-2) \times d_x}
\end{bmatrix}
,~
\tilde b_1 = 
\begin{bmatrix}
\veccomp{\bar W}{d_x+1}-\beta+\gamma\\
-\veccomp{\bar W}{d_x+1}+\beta+\gamma\\
\zeros{d_1-2}
\end{bmatrix}
,\\
&\tilde W_2 = \frac{1}{s_++s_-}  \begin{bmatrix} 1 & -1 & \zeros{d_1-2}^T \end{bmatrix}
, ~
\tilde b_2 = \beta.
\end{align*}
With these parameter values,
\begin{equation*}
\tilde W_1 x_i + \tilde b_1 = 
\begin{bmatrix}
\bar y_i - \alpha v^T x_i - \beta + \gamma\\
-\bar y_i + \alpha v^T x_i + \beta + \gamma\\
\zeros{d_1-2}
\end{bmatrix}.
\end{equation*}
As we saw in Lemma~\ref{lem:roleofbeta}, for $i \leq j_1$, $\bar y_i - \alpha v^T x_i - \beta + \gamma < 0$ and $- \bar y_i +\alpha v^T x_i + \beta + \gamma > 0$. So 
\begin{align*}
\hat y_i &= 
\tilde W_2 \bar h_{s_+,s_-} (\tilde W_1 x_i + \tilde b_1) + \tilde b_2 \\
&= \frac{1}{ s_++s_-} s_-  ( \bar y_i - \alpha v^T x_i - \beta + \gamma ) 
- \frac{1}{ s_++s_-} s_+ (- \bar y_i +\alpha v^T x_i + \beta + \gamma )
+ \beta\\
&= \bar y_i -\alpha v^T x_i - \frac{s_+ - s_-}{s_+ + s_-} \gamma.
\end{align*}
Similarly, for $i > j_1$, $\bar y_i - \alpha v^T x_i - \beta + \gamma > 0$ and $- \bar y_i +\alpha v^T x_i + \beta + \gamma < 0$, so
\begin{align*}
\hat y_i = \tilde W_2 \bar h_{s_+,s_-} (\tilde W_1 x_i + \tilde b_1) + \tilde b_2
= \bar y_i -\alpha v^T x_i + \frac{s_+ - s_-}{s_+ + s_-} \gamma.
\end{align*}
Now, the squared error loss of this point is
\begin{align*}
&\ell((\tilde W_j, \tilde b_j)_{j=1}^2) = \bhalf \nfro{\hat Y - Y}^2\\
=& \bhalf \sum_{i \leq j_1} \left (\bar y_i -\alpha v^T x_i - \frac{s_+ - s_-}{s_+ + s_-} \gamma - y_i \right )^2 
+ \bhalf \sum_{i > j_1} \left (\bar y_i -\alpha v^T x_i + \frac{s_+ - s_-}{s_+ + s_-} \gamma - y_i \right )^2\\
=& \bhalf \sum_{i=1}^m \left (\bar y_i - \alpha v^T x_i - y_i \right )^2
+ \left [ \sum_{i > j_1} \left (\bar y_i - \alpha  v^T x_i - y_i \right ) 
- \sum_{i \leq j_1} \left (\bar y_i - \alpha v^T x_i - y_i \right ) \right ] \frac{s_+ - s_-}{s_+ + s_-} \gamma
+ O(\gamma^2)\\
=& \ell_0(\bar W) 
- \alpha \left [ \sum_{i=1}^m \left ( \bar y_i - y_i \right ) x_i^T \right ] v 
+ O(\alpha^2) 
+ \left [ \sum_{i > j_1} \left (\bar y_i - y_i \right ) 
- \sum_{i \leq j_1} \left (\bar y_i - y_i \right ) \right ] \frac{s_+ - s_-}{s_+ + s_-} \gamma 
+ O(\alpha\gamma)
+ O(\gamma^2).
\end{align*}
Recall that $\sum_{i=1}^m \left ( \bar y_i - y_i \right ) x_i^T = 0$ for least squares estimates $\bar y_i$.
From Lemma~\ref{lem:roleofbeta}, we saw that
$\sum_{i > j_1} \left (\bar y_i - y_i \right ) - \sum_{i \leq j_1} \left (\bar y_i - y_i \right ) = -2 (\bar y_{j_1} - y_{j_1})$.
As seen in the definition of $\gamma$ \eqref{eq:c2gammadef}, the magnitude of $\gamma$ is proportional to $\alpha$. Substituting \eqref{eq:c2gammadef}, we can express the loss as
\begin{equation*}
\ell((\tilde W_j, \tilde b_j)_{j=1}^2)
= 
\ell_0(\bar W) 
-  \frac{\alpha | (\bar y_{j_1} - y_{j_1} ) (s_+ - s_-) | v^T(x_{j_1} - x_{j_2}) }{2(s_+ + s_-)}
+ O(\alpha^2).
\end{equation*}
Recall that $v^T(x_{j_1} - x_{j_2}) > 0$ from \eqref{eq:j2def}. Then, for sufficiently small $\alpha \in (0,1]$, 
\begin{equation*}
\ell((\tilde W_j, \tilde b_j)_{j=1}^2) < \ell_0(\bar W) = \ell((\hat W_j, \hat b_j)_{j=1}^2),
\end{equation*}
therefore proving that $(\hat W_j, \hat b_j)_{j=1}^2$ is a spurious local minimum.

\section{Proof of Theorem~\ref{thm:othernonlin}}
\label{sec:thm3}
\subsection{Proof of Part 1}
Given $v_1, v_2, v_3, v_4 \in \reals$ satisfying conditions~(C\ref{thm:othernonlin}.1) and (C\ref{thm:othernonlin}.2), we can pick parameter values $(\tilde W_j, \tilde b_j)_{j=1}^2$ to perfectly fit the given dataset:
\begin{align*}
&\tilde W_1 = 
\begin{bmatrix}
v_1 & v_2 \\
v_3 & v_4
\end{bmatrix}
,~
\tilde b_1 = 
\begin{bmatrix}
0\\
0
\end{bmatrix}
,~
\tilde W_2 = 
\left (h(v_3) h \left( \tfrac{v_1+v_2}{2} \right ) -\!h(v_1) h \left( \tfrac{v_3+v_4}{2} \right)\right)^{-1}
[h(v_3) -\!h(v_1)]
,~
\tilde b_2 = 0.
\end{align*}
With these values, we can check that $\hat Y = \begin{bmatrix} 0 & 0 & 1 \end{bmatrix}$,
hence perfectly fitting $Y$, thus the loss $\ell( (\tilde W_j, \tilde b_j)_{j=1}^2 ) = 0$.

\subsection{Proof of Part 2}
Given conditions (C\ref{thm:othernonlin}.3)--(C\ref{thm:othernonlin}.7) on $v_1, v_2, u_1, u_2 \in \reals$, we prove below that there exists a local minimum $(\hat W_j, \hat b_j)_{j=1}^2$ for which  
the output of the network is the same as linear least squares model, and its empirical risk is $\ell((\hat W_j, \hat b_j)_{j=1}^2) = \frac{1}{3}$. If the conditions of Part 1 also hold, this local minimum is strictly inferior to the global one. 

First, compute the output $\bar Y$ of linear least squares model to obtain
$\bar Y = \begin{bmatrix} \frac{1}{3} & \frac{1}{3} & \frac{1}{3} \end{bmatrix}$. Now assign parameter values
\begin{align*}
\hat W_1 = 
\begin{bmatrix}
v_1 & v_1 \\
v_2 & v_2
\end{bmatrix}
,~
\hat b_1 = 
\begin{bmatrix}
0\\
0
\end{bmatrix}
,~
\hat W_2 = 
\begin{bmatrix} u_1 & u_2 \end{bmatrix}
,~
\hat b_2 = 0.
\end{align*}
With these values we can check that $\hat Y = \begin{bmatrix} \frac{1}{3} & \frac{1}{3} & \frac{1}{3} \end{bmatrix}$,
under condition (C\ref{thm:othernonlin}.3): $u_1 h(v_1) + u_2 h(v_2) = \frac{1}{3}$.
The empirical risk is $\ell((\hat W_j, \hat b_j)_{j=1}^2) = \half ( \frac{1}{9} + \frac{1}{9} + \frac{4}{9} ) = \frac{1}{3}$.

It remains to show that this is indeed a local minimum of $\ell$. 
To show this, we apply perturbations to the parameters 
to see if the risk after perturbation is greater than or equal to $\ell((\hat W_j, \hat b_j)_{j=1}^2)$.
Let the perturbed parameters be
\begin{equation}
\label{eq:2}
\begin{split}
&\check W_1 = 
\begin{bmatrix}
v_1+\delta_{11} & v_1+\delta_{12} \\
v_2+\delta_{21} & v_2+\delta_{22}
\end{bmatrix}
,~
\check b_1 = 
\begin{bmatrix}
\beta_1\\
\beta_2
\end{bmatrix}
,~
\check W_2 = 
\begin{bmatrix} u_1+\epsilon_1 & u_2+\epsilon_2 \end{bmatrix}
,~
\check b_2 = \gamma,
\end{split}
\end{equation}
where $\delta_{11}, \delta_{12}, \delta_{21}, \delta_{22}, \beta_1, \beta_2, \epsilon_1, \epsilon_2,$ and $\gamma$ are small real numbers.
The next lemma rearranges the terms in $\ell((\check W_j, \check b_j)_{j=1}^2)$ into a form that helps us prove local minimality of $(\hat W_j, \hat b_j)_{j=1}^2$. Appendix~\ref{sec:thm3techlem} gives the proof of Lemma~\ref{lem:thm3techlem}, which includes as a byproduct some equalities on polynomials that may be of wider interest. 
\begin{lemma}
	\label{lem:thm3techlem}
	Assume there exist real numbers $v_1, v_2, u_1, u_2$ such that conditions~(C\ref{thm:othernonlin}.3)--(C\ref{thm:othernonlin}.5) hold. Then, for perturbed parameters $(\check W_j, \check b_j)_{j=1}^2$ defined in~\eqref{eq:2},
	\begin{align*}
	\ell((\check W_j, \check b_j)_{j=1}^2) \geq 
	\tfrac{1}{3} 
	+ \alpha_1 (\delta_{11}-\delta_{12})^2
	+ \alpha_2 (\delta_{21}-\delta_{22})^2
	+ \alpha_3 (\delta_{11}\!-\!\delta_{12})(\delta_{21}\!-\!\delta_{22}),
	\numberthis \label{eq:erisk3m}
	\end{align*}
	where $\alpha_i = \frac{u_i h''(v_i)}{12} + \frac{u_i^2 (h'(v_i))^2}{4} + o(1)$, for $i = 1,2$, 
	and $\alpha_3 = \frac{u_1 u_2 h'(v_1)h'(v_2)}{2} + o(1)$,
	and $o(1)$ contains terms that diminish to zero as perturbations vanish.
\end{lemma}

To make the the sum of the last three terms of \eqref{eq:erisk3m} nonnegative, 
we need to satisfy $\alpha_1 \geq 0$ and $\alpha_3^2 - 4\alpha_1 \alpha_2 \leq 0$; these inequalities  are satisfied for small enough perturbations because of conditions~(C\ref{thm:othernonlin}.6)--(C\ref{thm:othernonlin}.7).
Thus, we conclude that $\ell((\check W_j, \check b_j)_{j=1}^2) \geq \frac{1}{3}  = \ell((\hat W_j, \hat b_j)_{j=1}^2)$ for small enough perturbations, proving that $(\hat W_j, \hat b_j)_{j=1}^2$ is a local minimum.

\section{Proof of Lemma~\ref{lem:thm3techlem}}
\label{sec:thm3techlem}
The goal of this lemma is to prove that
\begin{align*}
\ell((\check W_j, \check b_j)_{j=1}^2) =
&\frac{1}{3} + \frac{3}{2}\text{(perturbations)}^2
+ \left (\frac{u_1 h''(v_1)}{12} + \frac{u_1^2 (h'(v_1))^2}{4} + o(1) \right ) (\delta_{11}-\delta_{12})^2 \\
+&\left (\frac{u_2 h''(v_2)}{12} + \frac{u_2^2 (h'(v_2))^2}{4} + o(1) \right ) (\delta_{21}-\delta_{22})^2 \\
+&\left (\frac{u_1 u_2 h'(v_1)h'(v_2)}{2} + o(1) \right) (\delta_{11}-\delta_{12})(\delta_{21}-\delta_{22}),
\numberthis \label{eq:erisk3}
\end{align*}
where $o(1)$ contains terms that diminish to zero as perturbations decrease.

Using the perturbed parameters,
\begin{equation*}
\check W_1 X + \check b_1 \ones{m}^T = 
\begin{bmatrix}
v_1 + \delta_{11} + \beta_1 & v_1 + \delta_{12} + \beta_1 & v_1 + \frac{\delta_{11}+\delta_{12}}{2} + \beta_1 \\
v_2 + \delta_{21} + \beta_2 & v_2 + \delta_{22} + \beta_2 & v_2 + \frac{\delta_{21}+\delta_{22}}{2} + \beta_2
\end{bmatrix},
\end{equation*}
so the empirical risk can be expressed as
\begin{align*}
&\ell((\check W_j, \check b_j)_{j=1}^2)\\
=& \bhalf \nfro{\check W_2 h \left (\check W_1 X + \check b_1 \ones{m}^T \right ) + \check b_2 \ones{m}^T - Y}^2\\
=& \bhalf \left [ (u_1 + \epsilon_1) h (v_1 + \delta_{11} + \beta_1) + (u_2 + \epsilon_2) h (v_2 + \delta_{21} + \beta_2) + \gamma
\right ]^2 \\
&+ \bhalf \left [ (u_1 + \epsilon_1) h (v_1 + \delta_{12} + \beta_1) + (u_2 + \epsilon_2) h (v_2 + \delta_{22} + \beta_2) + \gamma
\right ]^2 \\      
&+ \bhalf \left [ (u_1 + \epsilon_1) h \left (v_1 + \frac{\delta_{11}+\delta_{12}}{2} + \beta_1 \right) + 
(u_2 + \epsilon_2) h \left (v_2 + \frac{\delta_{21}+\delta_{22}}{2} + \beta_2 \right ) + \gamma - 1
\right ]^2 \numberthis \label{eq:erisk1}
\end{align*}
So, the empirical risk \eqref{eq:erisk1} consists of three terms, one for each training example.
By expanding the activation function $h$ using Taylor series expansion and doing algebraic manipulations, we will derive the equation \eqref{eq:erisk3} from \eqref{eq:erisk1}.

Using the Taylor series expansion, we can express $h(v_1 + \delta_{11} + \beta_1)$ as
\begin{equation*}
h(v_1 + \delta_{11} + \beta_1) = h(v_1) + \sum_{n=1}^{\infty} \frac{h^{(n)}(v_1)}{n!} (\delta_{11}+\beta_1)^n.
\end{equation*}
Using a similar expansion for $h(v_2 + \delta_{21} + \beta_2)$, the first term of \eqref{eq:erisk1} can be written as
\begin{align*}
&\bhalf \left [ (u_1 + \epsilon_1) h(v_1 + \delta_{11} + \beta_1) + (u_2 + \epsilon_2) h(v_2 + \delta_{21} + \beta_2) + \gamma
\right ]^2 \\
=&\bhalf \Bigg [ (u_1 + \epsilon_1) \left ( h(v_1) + \sum_{n=1}^{\infty} \frac{h^{(n)}(v_1)}{n!} (\delta_{11}+\beta_1)^n \right)
+ (u_2 + \epsilon_2) \left ( h(v_2) + \sum_{n=1}^{\infty} \frac{h^{(n)}(v_2)}{n!} (\delta_{21}+\beta_2)^n  \right) 
+ \gamma
\Bigg ]^2 \\
=&\bhalf \Bigg [ \frac{1}{3} + \epsilon_1 h(v_1) + (u_1 + \epsilon_1) \sum_{n=1}^{\infty} \frac{h^{(n)}(v_1)}{n!} (\delta_{11}+\beta_1)^n
+ \epsilon_2 h(v_2) + (u_2 + \epsilon_2) \sum_{n=1}^{\infty} \frac{h^{(n)}(v_2)}{n!} (\delta_{21}+\beta_2)^n
+ \gamma
\Bigg ]^2,
\end{align*}
where we used $u_1 h(v_1) + u_2 h(v_2) = \frac{1}{3}$.
To simplify notation, let us introduce the following function:
\begin{align*}
t(\delta_1, \delta_2) 
=  \epsilon_1 h(v_1) + \epsilon_2 h(v_2) + \gamma
+ (u_1 + \epsilon_1) \sum_{n=1}^{\infty}\frac{h^{(n)}(v_1)}{n!} (\delta_{1}+\beta_1)^n
+ (u_2 + \epsilon_2) \sum_{n=1}^{\infty} \frac{h^{(n)}(v_2)}{n!} (\delta_{2}+\beta_2)^n.
\end{align*}
With this new notation $t(\delta_1, \delta_2)$, after doing similar expansions to the other terms of \eqref{eq:erisk1}, we get
\begin{align*}
&\ell((\check W_j, \check b_j)_{j=1}^2)\\
=& \bhalf \left [ \frac{1}{3} + t(\delta_{11}, \delta_{21}) \right ]^2
+ \bhalf \left [ \frac{1}{3} + t(\delta_{12}, \delta_{22}) \right ]^2
+ \bhalf \left [ -\frac{2}{3} + t\left (\frac{\delta_{11}+\delta_{12}}{2}, \frac{\delta_{21}+\delta_{22}}{2} \right )  \right ]^2\\
= &\frac{1}{3} 
+ \frac{1}{3} \left [ t(\delta_{11}, \delta_{21}) + t(\delta_{12}, \delta_{22}) - 2 t\left (\frac{\delta_{11}+\delta_{12}}{2}, \frac{\delta_{21}+\delta_{22}}{2} \right ) \right ]\\
&+ \bhalf \left [ t(\delta_{11}, \delta_{21}) \right ]^2 
+ \bhalf \left [ t(\delta_{12}, \delta_{22})  \right ]^2 
+ \bhalf \left [ t\left (\frac{\delta_{11}+\delta_{12}}{2}, \frac{\delta_{21}+\delta_{22}}{2} \right ) \right ]^2
\numberthis \label{eq:erisk2}
\end{align*}

Before we show the lower bounds, we first present the following lemmas that will prove useful shortly. These are simple yet interesting lemmas that might be of independent interest.
\begin{lemma}
	\label{lem:algebralem1}
	For $n \geq 2$,
	\begin{equation*}
	a^n + b^n - 2 \left ( \frac{a+b}{2} \right )^n = (a-b)^2 p_{n}(a,b),
	\end{equation*}
	where $p_{n}$ is a polynomial in $a$ and $b$.
	All terms in $p_n$ have degree exactly $n-2$.
	When $n=2$, $p_2(a,b) = \bhalf$.
\end{lemma}
\begin{proof}
	The exact formula for $p_n(a,b)$ is as the following:
	\begin{equation*}
	p_n(a,b) = \sum_{k=0}^{n-2} \left [ k+1 - 2^{-n+1} \sum_{l=0}^k (k+1-l) \choose{n}{l} \right ] a^{n-k-2} b^k.
	\end{equation*}
	Using this, we can check the lemma is correct just by expanding both sides of the equation.
	The rest of the proof is straightforward but involves some complicated algebra. So, we omit the details for simplicity.
\end{proof}
\begin{lemma}
	\label{lem:algebralem2}
	For $n_1, n_2 \geq 1$,
	\begin{align*}
	&a^{n_1} c^{n_2} + b^{n_1} d^{n_2} - 2 \left ( \frac{a+b}{2} \right )^{n_1} \left ( \frac{c+d}{2} \right )^{n_2}\\
	=& (a-b)^2 q_{n_1,n_2}(a,b,d) + (c-d)^2 q_{n_2,n_1}(c,d,b) + (a-b)(c-d) r_{n_1,n_2} (a,b,c,d)
	\end{align*}
	where $q_{n_1,n_2}$ and $r_{n_1,n_2}$ are polynomials in $a$, $b$, $c$ and $d$. 
	All terms in $q_{n_1,n_2}$ and $r_{n_1,n_2}$ have degree exactly $n_1 + n_2 - 2$.
	When $n_1= n_2 = 1$, $q_{1,1}(a,b,d) = 0$ and $r_{1,1} (a,b,c,d) = \bhalf$.
\end{lemma}
\begin{proof}
	The exact formulas for $q_{n_1,n_2}(a,b,d)$, $q_{n_2,n_1}(c,d,b)$, and $r_{n_1,n_2} (a,b,c,d)$ are as the following:
	\begin{align*}
	&q_{n_1,n_2}(a,b,d) = \sum_{k_1=0}^{n_1-2} \left [ k_1+1 - 2^{-n_1+1} \sum_{l_1=0}^{k_1} (k_1+1-l_1) \choose{n_1}{l_1} \right ] a^{n_1-k_1-2} b^{k_1} d^{n_2},\\
	&q_{n_2,n_1}(c,d,b) = \sum_{k_2=0}^{n_2-2} \left [ k_2+1 - 2^{-n_2+1} \sum_{l_2=0}^{k_2} (k_2+1-l_2) \choose{n_2}{l_2} \right ] b^{n_1} c^{n_2-k_2-2} d^{k_2},\\
	&r_{n_1,n_2}(a,b,c,d) = \sum_{k_1=0}^{n_1-1} \sum_{k_2=0}^{n_2-1} \left [ 1 - 2^{-n_1 - n_2 + 1} \sum_{l_1=0}^{k_1} \sum_{l_2 = 0}^{k_2} \choose{n_1}{l_1} \choose{n_2}{l_2} \right ] a^{n_1-k_1-1} b^{k_1} c^{n_2-k_2-1} d^{k_2}.
	\end{align*}
	Similarly, we can check the lemma is correct just by expanding both sides of the equation.
	The remaining part of the proof is straightforward, so we will omit the details.
\end{proof}

Using Lemmas~\ref{lem:algebralem1} and \ref{lem:algebralem2}, we will expand and simplify the ``cross terms'' part and ``squared terms'' part of \eqref{eq:erisk2}.
For the ``cross terms'' in \eqref{eq:erisk2}, let us split $t(\delta_1, \delta_2)$ into two functions $t_1$ and $t_2$:
\begin{align*}
t_1(\delta_1, \delta_2) 
= &  \epsilon_1 h(v_1) + \epsilon_2 h(v_2) + \gamma + (u_1 + \epsilon_1) h'(v_1) (\delta_1 + \beta_1) + (u_2 + \epsilon_2) h'(v_2) (\delta_2 + \beta_2)\\
t_2(\delta_1, \delta_2) 
= & (u_1 + \epsilon_1) \sum_{n=2}^{\infty}\frac{h^{(n)}(v_1)}{n!} (\delta_{1}+\beta_1)^n
+ (u_2 + \epsilon_2) \sum_{n=2}^{\infty} \frac{h^{(n)}(v_2)}{n!} (\delta_{2}+\beta_2)^n,
\end{align*}
so that $t(\delta_1, \delta_2) = t_1(\delta_1, \delta_2) + t_2(\delta_1,\delta_2)$.
It is easy to check that
\begin{equation*}
t_1(\delta_{11}, \delta_{21}) + t_1(\delta_{12}, \delta_{22}) - 2 t_1\left (\frac{\delta_{11}+\delta_{12}}{2}, \frac{\delta_{21}+\delta_{22}}{2} \right ) = 0.
\end{equation*}
Also, using Lemma~\ref{lem:algebralem1}, we can see that
\begin{align*}
(\delta_{11} + \beta_1)^n + (\delta_{12} + \beta_1)^n - 2 \left ( \frac{\delta_{11} + \delta_{12}}{2} + \beta_1 \right )^n 
= (\delta_{11}-\delta_{12})^2 p_{n}(\delta_{11} + \beta_1,\delta_{12} + \beta_1),\\
(\delta_{21} + \beta_2)^n + (\delta_{22} + \beta_2)^n - 2 \left ( \frac{\delta_{21} + \delta_{22}}{2} + \beta_2 \right )^n 
= (\delta_{21}-\delta_{22})^2 p_{n}(\delta_{21} + \beta_2,\delta_{22} + \beta_2),
\end{align*}
so
\begin{align*}
&t_2(\delta_{11}, \delta_{21}) + t_2(\delta_{12}, \delta_{22}) - 2 t_2\left (\frac{\delta_{11}+\delta_{12}}{2}, \frac{\delta_{21}+\delta_{22}}{2} \right )\\
=& (u_1 + \epsilon_1) (\delta_{11}-\delta_{12})^2 \sum_{n=2}^{\infty}\frac{h^{(n)}(v_1)}{n!} p_{n}(\delta_{11} + \beta_1,\delta_{12} + \beta_1)\\
&+(u_2 + \epsilon_2) (\delta_{21}-\delta_{22})^2 \sum_{n=2}^{\infty}\frac{h^{(n)}(v_2)}{n!} p_{n}(\delta_{21} + \beta_2,\delta_{22} + \beta_2).
\end{align*}
Consider the summation
\begin{equation*}
\sum_{n=2}^{\infty}\frac{h^{(n)}(v_1)}{n!} p_{n}(\delta_{11} + \beta_1,\delta_{12} + \beta_1).
\end{equation*}
We assumed that there exists a constant $c > 0$ such that $|h^{(n)}(v_1)| \leq c^n n!$.
From this, for small enough perturbations $\delta_{11}$, $\delta_{12}$, and $\beta_1$, 
we can see that the summation converges, and the summands converge to zero as $n$ increases.
Because all the terms in $p_n$ ($n \geq 3$) are of degree at least one, we can thus write
\begin{equation*}
\sum_{n=2}^{\infty}\frac{h^{(n)}(v_1)}{n!} p_{n}(\delta_{11} + \beta_1,\delta_{12} + \beta_1) = \frac{h''(v_1)}{4} + o(1).
\end{equation*}
So, for small enough $\delta_{11}$, $\delta_{12}$, and $\beta_1$, the term $\frac{h''(v_1)}{4}$ dominates the summation. 
Similarly, as long as $\delta_{21}$, $\delta_{22}$, and $\beta_2$ are small enough, 
the summation $\sum_{n=2}^{\infty}\frac{h^{(n)}(v_2)}{n!} p_{n}(\delta_{21} + \beta_2,\delta_{22} + \beta_2)$ is dominated by $\frac{h''(v_2)}{4}$.
In conclusion, for small enough perturbations, 
\begin{align*}
&t(\delta_{11}, \delta_{21}) + t(\delta_{12}, \delta_{22}) - 2 t\left (\frac{\delta_{11}+\delta_{12}}{2}, \frac{\delta_{21}+\delta_{22}}{2} \right ) \\
=& t_2(\delta_{11}, \delta_{21}) + t_2(\delta_{12}, \delta_{22}) - 2 t_2\left (\frac{\delta_{11}+\delta_{12}}{2}, \frac{\delta_{21}+\delta_{22}}{2} \right ) \\
=& (u_1 + o(1)) \left (\frac{h''(v_1)}{4} + o(1) \right ) (\delta_{11}-\delta_{12})^2 
+ (u_2 + o(1)) \left (\frac{h''(v_2)}{4} + o(1) \right ) (\delta_{21}-\delta_{22})^2\\
=& \left (\frac{u_1 h''(v_1)}{4} + o(1) \right ) (\delta_{11}-\delta_{12})^2 + \left (\frac{u_2 h''(v_2)}{4} + o(1) \right ) (\delta_{21}-\delta_{22})^2. \numberthis \label{eq:erisk2p1}
\end{align*}

Now, it is time to take care of the ``squared terms.'' We will express the terms as 
\begin{align*}
&\bhalf \left [ t(\delta_{11}, \delta_{21}) \right ]^2 
+ \bhalf \left [ t(\delta_{12}, \delta_{22})  \right ]^2 
+ \bhalf \left [ t\left (\frac{\delta_{11}+\delta_{12}}{2}, \frac{\delta_{21}+\delta_{22}}{2} \right ) \right ]^2\\
= &\frac{3}{2} \left [ t\left (\frac{\delta_{11}+\delta_{12}}{2}, \frac{\delta_{21}+\delta_{22}}{2} \right ) \right ]^2
+ \bhalf \left [ t(\delta_{11}, \delta_{21}) \right ]^2 + \bhalf \left [ t(\delta_{12}, \delta_{22})  \right ]^2  - \left [ t\left (\frac{\delta_{11}+\delta_{12}}{2}, \frac{\delta_{21}+\delta_{22}}{2} \right ) \right ]^2,\numberthis \label{eq:erisk2p2}
\end{align*}
and expand and simplify the terms in 
\begin{equation*}
\bhalf \left [ t(\delta_{11}, \delta_{21}) \right ]^2 + \bhalf \left [ t(\delta_{12}, \delta_{22})  \right ]^2  - \left [ t\left (\frac{\delta_{11}+\delta_{12}}{2}, \frac{\delta_{21}+\delta_{22}}{2} \right ) \right ]^2.
\end{equation*}
This time, we split $t(\delta_1, \delta_2)$ in another way, this time into three parts:
\begin{align*}
&t_3 = \epsilon_1 h(v_1) + \epsilon_2 h(v_2) + \gamma,\\
&t_4(\delta_1) = (u_1 + \epsilon_1) \sum_{n=1}^{\infty}\frac{h^{(n)}(v_1)}{n!} (\delta_{1}+\beta_1)^n,\\
&t_5(\delta_2) = (u_2 + \epsilon_2) \sum_{n=1}^{\infty} \frac{h^{(n)}(v_2)}{n!} (\delta_{2}+\beta_2)^n,
\end{align*}
so that $t(\delta_1, \delta_2) = t_3 + t_4(\delta_1) + t_5(\delta_2)$. With this,
\begin{align*}
&\bhalf \left [ t(\delta_{11}, \delta_{21}) \right ]^2 + \bhalf \left [ t(\delta_{12}, \delta_{22})  \right ]^2  
- \left [ t\left (\frac{\delta_{11}+\delta_{12}}{2}, \frac{\delta_{21}+\delta_{22}}{2} \right ) \right ]^2\\
=& t_3 \left [ t_4(\delta_{11}) + t_4(\delta_{12}) - 2 t_4 \left (\frac{\delta_{11}+\delta_{12}}{2} \right) 
+ t_5(\delta_{21}) + t_5(\delta_{22}) - 2 t_5 \left (\frac{\delta_{21}+\delta_{22}}{2} \right) \right ]\\
& + \bhalf \left [ (t_4(\delta_{11}))^2 + (t_4(\delta_{12}))^2 - 2 \left ( t_4 \left (\frac{\delta_{11}+\delta_{12}}{2} \right) \right) ^2 \right ]\\
& + \bhalf \left [ (t_5(\delta_{21}))^2 + (t_5(\delta_{22}))^2 - 2 \left ( t_5 \left (\frac{\delta_{21}+\delta_{22}}{2} \right) \right) ^2 \right ]\\
& + \left [ t_4(\delta_{11})t_5(\delta_{21}) + t_4(\delta_{12}) t_5(\delta_{22}) - 2 t_4 \left (\frac{\delta_{11}+\delta_{12}}{2} \right) t_5 \left (\frac{\delta_{21}+\delta_{22}}{2} \right) \right ]. \numberthis \label{eq:sqdiff}
\end{align*}
We now have to simplify the equation term by term. We first note that 
\begin{align*}
&t_4(\delta_{11}) + t_4(\delta_{12}) - 2 t_4 \left (\frac{\delta_{11}+\delta_{12}}{2} \right) 
+ t_5(\delta_{21}) + t_5(\delta_{22}) - 2 t_5 \left (\frac{\delta_{21}+\delta_{22}}{2} \right)\\
=& t_2(\delta_{11}, \delta_{21}) + t_2(\delta_{12}, \delta_{22}) - 2 t_2\left (\frac{\delta_{11}+\delta_{12}}{2}, \frac{\delta_{21}+\delta_{22}}{2} \right ),
\end{align*}
so 
\begin{align*}
&t_3 \left [ t_4(\delta_{11}) + t_4(\delta_{12}) - 2 t_4 \left (\frac{\delta_{11}+\delta_{12}}{2} \right) 
+ t_5(\delta_{21}) + t_5(\delta_{22}) - 2 t_5 \left (\frac{\delta_{21}+\delta_{22}}{2} \right) \right ]\\
=& t_3 \left [t_2(\delta_{11}, \delta_{21}) + t_2(\delta_{12}, \delta_{22}) - 2 t_2\left (\frac{\delta_{11}+\delta_{12}}{2}, \frac{\delta_{21}+\delta_{22}}{2} \right ) \right ]\\
=& o(1) \left [ \left (\frac{u_1 h''(v_1)}{4} + o(1) \right ) (\delta_{11}-\delta_{12})^2 + \left (\frac{u_2 h''(v_2)}{4} + o(1) \right ) (\delta_{21}-\delta_{22})^2 \right ], \numberthis \label{eq:sqdiffp1}
\end{align*}
as seen in \eqref{eq:erisk2p1}.
Next, we have
\begin{align*}
&(t_4(\delta_{11}))^2 + (t_4(\delta_{12}))^2 - 2 \left ( t_4 \left (\frac{\delta_{11}+\delta_{12}}{2} \right) \right) ^2\\
= &(u_1 + \epsilon_1)^2 \sum_{n_1, n_2 = 1}^{\infty} \frac{h^{(n_1)}(v_1) h^{(n_2)}(v_1)}{n_1! n_2!}
\left [ (\delta_{11}+\beta_1)^{n_1+n_2} + (\delta_{12}+\beta_1)^{n_1+n_2} - 2\left (\frac{\delta_{11}+\delta_{12}}{2}+\beta_1 \right )^{n_1+n_2}\right ],\\
= &(u_1+\epsilon_1)^2 (\delta_{11}-\delta_{12})^2 \sum_{n_1, n_2 = 1}^{\infty} \frac{h^{(n_1)}(v_1) h^{(n_2)}(v_1)}{n_1! n_2!} p_{n_1+n_2} (\delta_{11}+\beta_1 , \delta_{12}+\beta_1)\\
= &\left ( \frac{u_1^2 (h'(v_1))^2}{2} + o(1) \right ) (\delta_{11}-\delta_{12})^2,\numberthis \label{eq:sqdiffp2}
\end{align*}
when perturbations are small enough. We again used Lemma~\ref{lem:algebralem1} in the second equality sign,
and the facts that $p_{n_1+n_2}(\cdot) = o(1)$ whenever $n_1+n_2>2$ and that $p_2(\cdot) = \bhalf$.
In a similar way,
\begin{equation}
(t_5(\delta_{21}))^2 + (t_5(\delta_{22}))^2 - 2 \left ( t_5 \left (\frac{\delta_{21}+\delta_{22}}{2} \right) \right) ^2
= \left ( \frac{u_2^2 (h'(v_2))^2}{2} + o(1) \right ) (\delta_{21}-\delta_{22})^2. \label{eq:sqdiffp3}
\end{equation}
Lastly,
\begin{align*}
&t_4(\delta_{11})t_5(\delta_{21}) + t_4(\delta_{12}) t_5(\delta_{22}) - 2 t_4 \left (\frac{\delta_{11}+\delta_{12}}{2} \right) t_5 \left (\frac{\delta_{21}+\delta_{22}}{2} \right)\\
= &(u_1 + \epsilon_1)(u_2 + \epsilon_2) \sum_{n_1, n_2 = 1}^{\infty} \frac{h^{(n_1)}(v_1) h^{(n_2)}(v_2)}{n_1! n_2!}
\Bigg [ (\delta_{11}+\beta_1)^{n_1} (\delta_{21}+\beta_2)^{n_2} \\
&              + (\delta_{12}+\beta_1)^{n_1} (\delta_{22}+\beta_2)^{n_2} 
- 2\left (\frac{\delta_{11}+\delta_{12}}{2}+\beta_1 \right )^{n_1} \left (\frac{\delta_{21}+\delta_{22}}{2}+\beta_2 \right )^{n_2} \Bigg ],\\
= &(u_1 + \epsilon_1)(u_2 + \epsilon_2) \Bigg [
(\delta_{11}-\delta_{12})^2 \sum_{n_1, n_2 = 1}^{\infty} \frac{h^{(n_1)}(v_1) h^{(n_2)}(v_2)}{n_1! n_2!} 
q_{n_1, n_2} (\delta_{11}+\beta_1, \delta_{12}+\beta_1, \delta_{22}+\beta_2)\\
&    +(\delta_{21}-\delta_{22})^2 \sum_{n_1, n_2 = 1}^{\infty} \frac{h^{(n_1)}(v_1) h^{(n_2)}(v_2)}{n_1! n_2!} 
q_{n_2, n_1} (\delta_{21}+\beta_2, \delta_{22}+\beta_2, \delta_{12}+\beta_1)\\
&    +(\delta_{11}-\delta_{12})(\delta_{21}-\delta_{22}) \sum_{n_1, n_2 = 1}^{\infty} \frac{h^{(n_1)}(v_1) h^{(n_2)}(v_2)}{n_1! n_2!} 
r_{n_1,n_2} (\delta_{11}+\beta_1, \delta_{12}+\beta_1, \delta_{21}+\beta_2, \delta_{22}+\beta_2) \Bigg ]\\
= &(u_1 u_2 + o(1)) \left [ (\delta_{11}-\delta_{12})^2 o(1) + (\delta_{21}-\delta_{22})^2 o(1) 
+ (\delta_{11}-\delta_{12})(\delta_{21}-\delta_{22}) \left (\frac{h'(v_1)h'(v_2)}{2} + o(1) \right) \right ], \numberthis \label{eq:sqdiffp4}
\end{align*}
where the second equality sign used Lemma~\ref{lem:algebralem2} and the third equality sign used 
the facts that $q_{n_1,n_2}(\cdot) = o(1)$ and $r_{n_1,n_2}(\cdot) = o(1)$ whenever $n_1+n_2 > 2$, and that $q_{1,1}(\cdot) = 0$ and $r_{1,1}(\cdot) = \bhalf$.
If we substitute \eqref{eq:sqdiffp1}, \eqref{eq:sqdiffp2}, \eqref{eq:sqdiffp3}, and \eqref{eq:sqdiffp4} into \eqref{eq:sqdiff},
\begin{align*}
&\bhalf \left [ t(\delta_{11}, \delta_{21}) \right ]^2 + \bhalf \left [ t(\delta_{12}, \delta_{22})  \right ]^2  
- \left [ t\left (\frac{\delta_{11}+\delta_{12}}{2}, \frac{\delta_{21}+\delta_{22}}{2} \right ) \right ]^2\\
=& o(1) \left [ \left (\frac{u_1 h''(v_1)}{4} + o(1) \right ) (\delta_{11}-\delta_{12})^2 + \left (\frac{u_2 h''(v_2)}{4} + o(1) \right ) (\delta_{21}-\delta_{22})^2 \right ]\\
& + \bhalf \left ( \frac{u_1^2 (h'(v_1))^2}{2} + o(1) \right ) (\delta_{11}-\delta_{12})^2
+ \bhalf \left ( \frac{u_2^2 (h'(v_2))^2}{2} + o(1) \right ) (\delta_{21}-\delta_{22})^2\\
& + (u_1 u_2 + o(1)) \left [ (\delta_{11}-\delta_{12})^2 o(1) + (\delta_{21}-\delta_{22})^2 o(1) 
+ (\delta_{11}-\delta_{12})(\delta_{21}-\delta_{22}) \left (\frac{h'(v_1)h'(v_2)}{2} + o(1) \right) \right ]\\
=& \left (\frac{u_1^2 (h'(v_1))^2}{4} + o(1) \right ) (\delta_{11}-\delta_{12})^2 + \left (\frac{u_2^2 (h'(v_2))^2}{4} + o(1) \right ) (\delta_{21}-\delta_{22})^2\\
& + \left (\frac{u_1 u_2 h'(v_1)h'(v_2)}{2} + o(1) \right) (\delta_{11}-\delta_{12})(\delta_{21}-\delta_{22}). \numberthis \label{eq:sqdiff2}
\end{align*}

We are almost done. If we substitute \eqref{eq:erisk2p1}, \eqref{eq:erisk2p2}, and \eqref{eq:sqdiff2} into \eqref{eq:erisk2}, we can get
\begin{align*}
&\ell((\check W_j, \check b_j)_{j=1}^2)\\
= &\frac{1}{3} + \frac{3}{2}\left [ t\left (\frac{\delta_{11}+\delta_{12}}{2}, \frac{\delta_{21}+\delta_{22}}{2} \right ) \right ]^2\\
&+ \left (\frac{u_1 h''(v_1)}{12} + o(1) \right ) (\delta_{11}-\delta_{12})^2 + \left (\frac{u_2 h''(v_2)}{12} + o(1) \right ) (\delta_{21}-\delta_{22})^2\\
& + \left (\frac{u_1^2 (h'(v_1))^2}{4} + o(1) \right ) (\delta_{11}-\delta_{12})^2 + \left (\frac{u_2^2 (h'(v_2))^2}{4} + o(1) \right ) (\delta_{21}-\delta_{22})^2\\
& + \left (\frac{u_1 u_2 h'(v_1)h'(v_2)}{2} + o(1) \right) (\delta_{11}-\delta_{12})(\delta_{21}-\delta_{22})\\
=& \frac{1}{3} + \frac{3}{2}\left [ t\left (\frac{\delta_{11}+\delta_{12}}{2}, \frac{\delta_{21}+\delta_{22}}{2} \right ) \right ]^2
+ \left (\frac{u_1 h''(v_1)}{12} + \frac{u_1^2 (h'(v_1))^2}{4} + o(1) \right ) (\delta_{11}-\delta_{12})^2 \\
&+\left (\frac{u_2 h''(v_2)}{12} + \frac{u_2^2 (h'(v_2))^2}{4} + o(1) \right ) (\delta_{21}-\delta_{22})^2
+\left (\frac{u_1 u_2 h'(v_1)h'(v_2)}{2} + o(1) \right) (\delta_{11}-\delta_{12})(\delta_{21}-\delta_{22}),
\end{align*}
which is the equation \eqref{eq:erisk3} that we were originally aiming to show.

\section{Proof of Corollary~\ref{cor:actftnex}}
\label{sec:cor2}
For the proof of this corollary, we present the values of real numbers that satisfy assumptions (C\ref{thm:othernonlin}.1)--(C\ref{thm:othernonlin}.7), for each activation function listed in the corollary:
sigmoid, tanh, arctan, exponential linear units (ELU, \cite{clevert2015fast}), scaled exponential linear units (SELU, \cite{klambauer2017self}).

To remind the readers what the assumptions were, we list the assumptions again.
For (C\ref{thm:othernonlin}.1)--(C\ref{thm:othernonlin}.2), 
there exist real numbers $v_1, v_2, v_3, v_4 \in \reals$ such that 
\begin{enumerate}[label=(C\ref{thm:othernonlin}.\arabic*), leftmargin=0.8in]
	\item $h(v_1) h(v_4) = h(v_2) h(v_3)$,
	\item $h(v_1) h \left ( \frac{v_3+v_4}{2} \right ) \neq h(v_3) h \left (\frac{v_1+v_2}{2} \right )$.
\end{enumerate}
For (C\ref{thm:othernonlin}.3)--(C\ref{thm:othernonlin}.7),
there exist real numbers $v_1, v_2, u_1, u_2 \in \reals$ such that the following assumptions hold:
\begin{enumerate}[label=(C\ref{thm:othernonlin}.\arabic*), leftmargin=0.8in]
	\setcounter{enumi}{2}
	\item $u_1 h(v_1) + u_2 h(v_2) = \frac{1}{3}$,
	\item $h$ is infinitely differentiable at $v_1$ and $v_2$,
	\item There exists a constant $c > 0$ such that $|h^{(n)}(v_1)| \leq c^n n!$ and $|h^{(n)}(v_2)| \leq c^n n!$.
	\item $(u_1 h'(v_1))^2 + \frac{u_1 h''(v_1)}{3} > 0$,
	\item $(u_1 h'(v_1) u_2 h'(v_2))^2 < ( (u_1 h'(v_1))^2 + \frac{u_1 h''(v_1)}{3} ) ( (u_2 h'(v_2))^2 + \frac{u_2 h''(v_2)}{3} )$.
\end{enumerate}
For each function, we now present the appropriate real numbers that satisfy the assumptions.

\subsection{Sigmoid}
When $h$ is sigmoid, 
\begin{equation*}
h (x) = \frac{1}{1+\exp(-x)}, ~ h^{-1}(x) = \log \left ( \frac{x}{1-x} \right ).
\end{equation*}
Assumptions (C\ref{thm:othernonlin}.1)--(C\ref{thm:othernonlin}.2) are satisfied by
\begin{equation*}
(v_1,v_2,v_3,v_4) = \left (h^{-1} \left ( \bhalf \right ), h^{-1}\left (\frac{1}{4}\right ), h^{-1}\left (\frac{1}{4}\right ), h^{-1}\left (\frac{1}{8}\right ) \right ),
\end{equation*}
and assumptions (C\ref{thm:othernonlin}.3)--(C\ref{thm:othernonlin}.7) are satisfied by
\begin{equation*}
(v_1,v_2,u_1,u_1) = \left (h^{-1} \left (\frac{1}{4}\right ), h^{-1}\left (\frac{1}{4}\right ), \frac{2}{3}, \frac{2}{3} \right ).
\end{equation*}
Among them, (C\ref{thm:othernonlin}.4)--(C\ref{thm:othernonlin}.5) follow because sigmoid function is an real analytic function \cite{krantz2002primer}.

\subsection{tanh}
When $h$ is hyperbolic tangent, 
assumptions (C\ref{thm:othernonlin}.1)--(C\ref{thm:othernonlin}.2) are satisfied by
\begin{equation*}
(v_1,v_2,v_3,v_4) = \left (\tanh^{-1} \left (\bhalf\right ), \tanh^{-1}\left (\frac{1}{4}\right ), \tanh^{-1}\left (\frac{1}{4}\right ), \tanh^{-1}\left (\frac{1}{8}\right ) \right ),
\end{equation*}
and assumptions (C\ref{thm:othernonlin}.3)--(C\ref{thm:othernonlin}.7) are satisfied by
\begin{equation*}
(v_1,v_2,u_1,u_1) = \left (\tanh^{-1} \left (\frac{1}{2}\right ), \tanh^{-1}\left (\frac{1}{2}\right ), 1, -\frac{1}{3} \right ),
\end{equation*}
Assumptions (C\ref{thm:othernonlin}.4)--(C\ref{thm:othernonlin}.5) hold because hyperbolic tangent function is real analytic.

\subsection{arctan}
When $h$ is inverse tangent, 
assumptions (C\ref{thm:othernonlin}.1)--(C\ref{thm:othernonlin}.2) are satisfied by
\begin{equation*}
(v_1,v_2,v_3,v_4) = \left (\tan \left (\bhalf\right ), \tan\left (\frac{1}{4}\right ), \tan\left (\frac{1}{4}\right ), \tan \left (\frac{1}{8}\right ) \right ),
\end{equation*}
and assumptions (C\ref{thm:othernonlin}.3)--(C\ref{thm:othernonlin}.7) are satisfied by
\begin{equation*}
(v_1,v_2,u_1,u_1) = \left (\tan \left (\frac{1}{2}\right ), \tan \left (\frac{1}{2}\right ), 1, -\frac{1}{3} \right ),
\end{equation*}
Assumptions (C\ref{thm:othernonlin}.4)--(C\ref{thm:othernonlin}.5) hold because inverse tangent function is real analytic.

\subsection{Quadratic}
When $h$ is quadratic, 
assumptions (C\ref{thm:othernonlin}.1)--(C\ref{thm:othernonlin}.2) are satisfied by
\begin{equation*}
(v_1,v_2,v_3,v_4) = \left (1, \frac{1}{2}, \frac{1}{2}, -\frac{1}{4} \right ),
\end{equation*}
and assumptions (C\ref{thm:othernonlin}.3)--(C\ref{thm:othernonlin}.7) are satisfied by
\begin{equation*}
(v_1,v_2,u_1,u_1) = \left (1, 1, \frac{1}{6}, \frac{1}{6} \right ),
\end{equation*}
Assumptions (C\ref{thm:othernonlin}.4)--(C\ref{thm:othernonlin}.5) hold because quadratic function is real analytic.

\subsection{ELU and SELU}
When $h$ is ELU or SELU,
\begin{equation*}
\begin{array}{ll}
h (x) = \lambda 
\begin{cases}
x & x \geq 0\\
\alpha(\exp(x)-1) & x < 0
\end{cases},~
&h^{-1} (x) =
\begin{cases}
x/\lambda & x \geq 0\\
\log \left ( \frac{x}{\lambda \alpha} + 1 \right ) & x < 0
\end{cases},\\
h' (x) = 
\begin{cases}
\lambda & x \geq 0\\
\lambda \alpha \exp(x) & x < 0
\end{cases},~
&h'' (x) =
\begin{cases}
0 & x \geq 0\\
\lambda \alpha \exp(x) & x < 0
\end{cases},
\end{array}
\end{equation*}
where $\alpha > 0$, and $\lambda = 1$ (ELU) or $\lambda > 1$ (SELU).
In this case, assumptions (C\ref{thm:othernonlin}.1)--(C\ref{thm:othernonlin}.2) are satisfied by
\begin{equation*}
(v_1,v_2,v_3,v_4) = \left (h^{-1} \left (-\frac{\lambda \alpha}{2}\right), h^{-1} \left (-\frac{\lambda \alpha}{4}\right), h^{-1} \left (-\frac{\lambda \alpha}{4}\right), h^{-1} \left (-\frac{\lambda \alpha}{8}\right) \right ).
\end{equation*}
Assumptions (C\ref{thm:othernonlin}.3)--(C\ref{thm:othernonlin}.7) are satisfied by
\begin{equation*}
(v_1,v_2,u_1,u_2) = \left (\frac{1}{3}, \log \left (\frac{2}{3}\right), \frac{2}{\lambda}, \frac{1}{\lambda \alpha} \right ),
\end{equation*}
where (C\ref{thm:othernonlin}.4)--(C\ref{thm:othernonlin}.5) are satisfied because $h(x)$ is real analytic at $v_1$ and $v_2$.

\section{Proof of Theorem~\ref{thm:othernonlin}  for ``ReLU-like'' activation functions.}
\label{sec:thm3forRelus}
Recall the piecewise linear nonnegative homogeneous activation function 
\begin{equation*}
\relulikeactfun(x) = 
\begin{cases}
s_+ x & x \geq 0\\
s_- x & x < 0,
\end{cases}
\end{equation*}
where $s_+>0$, $s_- \geq 0$ and $s_+ \neq s_-$, 
we will prove that the statements of Theorem~\ref{thm:othernonlin} hold for $\relulikeactfun$.

\subsection{Proof of Part 1}
In the case of $s_- > 0$, assumptions (C\ref{thm:othernonlin}.1)--(C\ref{thm:othernonlin}.2) are satisfied by
\begin{equation*}
(v_1,v_2,v_3,v_4) = \left (\frac{1}{s_+}, -\frac{1}{s_-}, -\frac{1}{s_-}, \frac{1}{s_+}\right) .
\end{equation*}
The rest of the proof can be done in exactly the same way as the proof of Theorem~\ref{thm:othernonlin}.1, provided in Appendix~\ref{sec:thm3}.

For $s_- = 0$, which corresponds to the case of ReLU, define parameters
\begin{align*}
&\tilde W_1 = 
\begin{bmatrix}
0 & 2\\
-2 & 1
\end{bmatrix}
,~
\tilde b_1 = 
\begin{bmatrix}
0\\
0
\end{bmatrix}
,~
\tilde W_2 = 
\begin{bmatrix}
\frac{1}{s_+} &-\frac{2}{s_+}
\end{bmatrix}
,~
\tilde b_2 = 0.
\end{align*}
We can check that 
\begin{equation*}
\relulikeactfun(\tilde W_1 X + \tilde b_1 \ones{3}^T) = 
s_+ \begin{bmatrix}
0 & 2 & 1 \\
0 & 1 & 0
\end{bmatrix},
\end{equation*}
so
\begin{equation*}
\tilde W_2\relulikeactfun(\tilde W_1 X + \tilde b_1 \ones{3}^T)+\tilde b_2 \ones{3}^T = 
\begin{bmatrix}
0 & 0 & 1
\end{bmatrix}.
\end{equation*}

\subsection{Proof of Part 2}
Assumptions (C\ref{thm:othernonlin}.3)--(C\ref{thm:othernonlin}.6) are satisfied by
\begin{equation*}
(v_1,v_2,u_1,u_1) = \left (\frac{1}{4 s_+}, \frac{1}{4 s_+}, \frac{2}{3}, \frac{2}{3}\right) .
\end{equation*}
Assign parameter values
\begin{align*}
\hat W_1 = 
\begin{bmatrix}
v_1 & v_1 \\
v_2 & v_2
\end{bmatrix}
,~
\hat b_1 = 
\begin{bmatrix}
0\\
0
\end{bmatrix}
,~
\hat W_2 = 
\begin{bmatrix} u_1 & u_2 \end{bmatrix}
,~
\hat b_2 = 0.
\end{align*}
It is easy to compute that the output of the neural network is $\hat Y = \begin{bmatrix} \frac{1}{3} & \frac{1}{3} & \frac{1}{3} \end{bmatrix}$,
so $\ell((\hat W_j, \hat b_j)_{j=1}^2) = \frac{1}{3}$.

Now, it remains to show that this is indeed a local minimum of $\ell$. 
To show this, we apply perturbations to the parameters 
to see if the risk after perturbation is greater than or equal to $\ell((\hat W_j, \hat b_j)_{j=1}^2)$.
Let the perturbed parameters be
\begin{align*}
\check W_1 = 
\begin{bmatrix}
v_1+\delta_{11} & v_1+\delta_{12} \\
v_2+\delta_{21} & v_2+\delta_{22}
\end{bmatrix}
,~
\check b_1 = 
\begin{bmatrix}
\beta_1\\
\beta_2
\end{bmatrix}
,
\check W_2 = 
\begin{bmatrix} u_1+\epsilon_1 & u_2+\epsilon_2 \end{bmatrix}
,~
\check b_2 = \gamma,
\end{align*}
where $\delta_{11}, \delta_{12}, \delta_{21}, \delta_{22}, \beta_1, \beta_2, \epsilon_1, \epsilon_2,$ and $\gamma$ are small enough real numbers.

Using the perturbed parameters,
\begin{equation*}
\check W_1 X + \check b_1 \ones{m}^T = 
\begin{bmatrix}
v_1 + \delta_{11} + \beta_1 & v_1 + \delta_{12} + \beta_1 & v_1 + \frac{\delta_{11}+\delta_{12}}{2} + \beta_1 \\
v_2 + \delta_{21} + \beta_2 & v_2 + \delta_{22} + \beta_2 & v_2 + \frac{\delta_{21}+\delta_{22}}{2} + \beta_2
\end{bmatrix},
\end{equation*}
so the empirical risk can be expressed as
\begin{align*}
&\ell((\check W_j, \check b_j)_{j=1}^2)\\
=& \bhalf \nfro{\check W_2 \bar h_{s_+,s_-} \left (\check W_1 X + \check b_1 \ones{m}^T \right ) + \check b_2 \ones{m}^T - Y}^2\\
=& \bhalf \left [ (u_1 + \epsilon_1) s_+ (v_1 + \delta_{11} + \beta_1) + (u_2 + \epsilon_2) s_+ (v_2 + \delta_{21} + \beta_2) + \gamma
\right ]^2 \\
&+ \bhalf \left [ (u_1 + \epsilon_1) s_+ (v_1 + \delta_{12} + \beta_1) + (u_2 + \epsilon_2) s_+ (v_2 + \delta_{22} + \beta_2) + \gamma
\right ]^2 \\      
&+ \bhalf \left [ (u_1 + \epsilon_1) s_+ \left (v_1 + \frac{\delta_{11}+\delta_{12}}{2} + \beta_1 \right) + 
(u_2 + \epsilon_2) s_+ \left (v_2 + \frac{\delta_{21}+\delta_{22}}{2} + \beta_2 \right ) + \gamma - 1
\right ]^2.
\end{align*}

To simplify notation, let us introduce the following function:
\begin{align*}
t(\delta_1, \delta_2) 
= s_+ \epsilon_1 v_1 + s_+ \epsilon_2 v_2 + \gamma + s_+ (u_1 + \epsilon_1) (\delta_1 + \beta_1) + s_+ (u_2 + \epsilon_2) (\delta_2 + \beta_2)
\end{align*}
It is easy to check that
\begin{equation*}
t(\delta_{11}, \delta_{21}) + t(\delta_{12}, \delta_{22}) - 2 t\left (\frac{\delta_{11}+\delta_{12}}{2}, \frac{\delta_{21}+\delta_{22}}{2} \right ) = 0.
\end{equation*}
With this new notation $t(\delta_1, \delta_2)$, we get
\begin{align*}
&\ell((\check W_j, \check b_j)_{j=1}^2)\\
=& \bhalf \left [ \frac{1}{3} + t(\delta_{11}, \delta_{21}) \right ]^2
+ \bhalf \left [ \frac{1}{3} + t(\delta_{12}, \delta_{22}) \right ]^2
+ \bhalf \left [ -\frac{2}{3} + t\left (\frac{\delta_{11}+\delta_{12}}{2}, \frac{\delta_{21}+\delta_{22}}{2} \right )  \right ]^2\\
= &\frac{1}{3} 
+ \frac{1}{3} \left [ t(\delta_{11}, \delta_{21}) + t(\delta_{12}, \delta_{22}) - 2 t\left (\frac{\delta_{11}+\delta_{12}}{2}, \frac{\delta_{21}+\delta_{22}}{2} \right ) \right ]\\
&+ \bhalf \left [ t(\delta_{11}, \delta_{21}) \right ]^2 
+ \bhalf \left [ t(\delta_{12}, \delta_{22})  \right ]^2 
+ \bhalf \left [ t\left (\frac{\delta_{11}+\delta_{12}}{2}, \frac{\delta_{21}+\delta_{22}}{2} \right ) \right ]^2
\geq \frac{1}{3} = \ell((\hat W_j, \hat b_j)_{j=1}^2).
\end{align*}

\section{Proof of Theorem~\ref{thm:linear}}
\label{sec:thm1}
Before we start, note the following partial derivatives, which can be computed using straightforward matrix calculus:
\begin{equation*}
\tfrac{\partial \ell}{\partial W_j} = (W_{H+1:j+1})^T \gradlzero(W_{H+1:1}) (W_{j-1:1})^T,
\end{equation*}
for all $j \in [H+1]$.

\subsection{Proof of Part 1, if $d_y \geq d_x$}
For Part 1, we must show that if $\gradlzero(\hat W_{H+1:1}) \neq 0$ then $(\hat W_j)_{j=1}^{H+1}$ is a saddle point of $\ell$. Thus, we show that $(\hat W_j)_{j=1}^{H+1}$ is neither a local minimum nor a local maximum. More precisely, for each $j$, let $\mc B_\epsilon (W_j)$ be an $\epsilon$-Frobenius-norm-ball centered at $W_j$, and $\prod_{j=1}^{H+1} \mc B_\epsilon(W_j)$ their Cartesian product. We wish to show that for every $\epsilon > 0$, there exist tuples $(P_j)_{j=1}^{H+1}$,  $(Q_j)_{j=1}^{H+1} \in \prod\nolimits_{j=1}^{H+1} \mc B_\epsilon(\hat W_j)$ such that
\begin{equation}
\label{eq:thm1p1stmt}
\ell( (P_j)_{j=1}^{H+1} ) > \ell((\hat W_j)_{j=1}^{H+1}) > \ell((Q_j)_{j=1}^{H+1}).
\end{equation}

To prove~\eqref{eq:thm1p1stmt}, we exploit $\ell( (\hat W_j)_{j=1}^{H+1} ) = \ell_0(\hat W_{H+1:1})$, and the assumption $\gradlzero(\hat W_{H+1:1}) \neq 0$. The key idea is to perturb the tuple $(\hat W_j)_{j=1}^{H+1}$ so that the directional derivative of $\ell_0$ along $P_{H+1:1} - \hat W_{H+1:1}$ is positive. 
Since $\ell_0$ is differentiable, if $P_{H+1:1} - \hat W_{H+1:1}$ is small, then
\begin{equation*}
\ell( (P_j)_{j=1}^{H+1} )\!=\!\ell_0(P_{H+1:1})\!>\!\ell_0(\hat W_{H+1:1})\!=\!\ell( (\hat W_j)_{j=1}^{H+1} ).
\end{equation*}
Similarly, we can show $\ell( (Q_j)_{j=1}^{H+1}) < \ell( (\hat W_j)_{j=1}^{H+1})$. The key challenge lies in constructing these perturbations; we outline our approach below; this construction may be of independent interest too.
For this section, we assume that $d_x \geq d_y$ for simplicity; the case $d_y \geq d_x$ is treated in Appendix~\ref{sec:thm1p1y}.

Since $\gradlzero(\hat W_{H+1:1}) \neq 0$, $\colsp(\gradlzero(\hat W_{H+1:1}))^\perp$ must be a strict subspace of $\reals^{d_y}$.
Consider $\partial \ell/\partial W_1$ at a critical point to see that $(\hat W_{H+1:2})^T \gradlzero(\hat W_{H+1:1}) = 0$, so $\colsp(\hat W_{H+1:2}) \subseteq \colsp(\gradlzero(\hat W_{H+1:1}))^\perp \subsetneq \reals^{d_y}$.
This strict inclusion implies $\rank(\hat W_{H+1:2}) < d_y \leq d_1$, so that 
$\nulsp(\hat W_{H+1:2})$ is not a trivial subspace. Moreover,  $\nulsp(\hat W_{H+1:2}) \supseteq \nulsp(\hat W_{H :2}) \supseteq \dots \supseteq \nulsp(\hat W_2)$.
We can split the proof into two cases: $\nulsp(\hat W_{H+1:2}) \neq \nulsp(\hat W_{H :2})$ and $\nulsp(\hat W_{H+1:2}) = \nulsp(\hat W_{H :2})$.

Let the SVD of $\gradlzero(\hat W_{H+1:1}) = U_l \Sigma U_r^T$.
Recall $\matcol{U_l}{1}$ and $\matcol{U_r}{1}$ denote first columns of $U_l$ and $U_r$, respectively.

\paragraph{Case 1: $\nulsp(\hat W_{H+1:2}) \neq \nulsp(\hat W_{H :2})$.}
In this case, $\nulsp(\hat W_{H+1:2}) \supsetneq \nulsp(\hat W_{H:2})$. 
We will perturb $\hat W_1$ and $\hat W_{H+1}$ to obtain the tuples $(P_j)_{j=1}^{H+1}$ and $(Q_j)_{j=1}^{H+1}$. 
To create our perturbation, we choose two unit vectors as follows:
\begin{equation*}
v_0 = \matcol{U_r}{1},~
v_1 \in \nulsp(\hat W_{H+1:2}) \cap \nulsp(\hat W_{H:2})^\perp.
\end{equation*}
Then, define $\Delta_1 \defeq \epsilon v_1 v_0^T \in \reals^{d_1 \times d_x}$, and $V_1 \defeq \hat W_1+\Delta_1 \in \mc B_\epsilon(\hat W_1)$.
Since $v_1$ lies in $\nulsp(\hat W_{H+1:2})$, observe that
\begin{equation*}
\hat W_{H+1:2} V_1 = \hat W_{H+1:1} + \epsilon \hat W_{H+1:2} v_1 v_0^T = \hat W_{H+1:1}.
\end{equation*}
With this definition of $V_1$, we can also see that
\begin{align*}
\gradlzero(\hat W_{H+1:1}) V_1^T (\hat W_{H:2})^T
= \gradlzero(\hat W_{H+1:1})  (\hat W_{H:1})^T + \epsilon \gradlzero(\hat W_{H+1:1}) v_0 v_1^T(\hat W_{H:2})^T.
\end{align*}
Note that $\gradlzero(\hat W_{H+1:1})  (\hat W_{H:1})^T$ is equal to $\partial \ell/\partial W_{H+1}$ at a critical point, hence is zero.
Since $v_0  = \matcol{U_r}{1}$, we have $\gradlzero(\hat W_{H+1:1}) v_0 = \sigmax(\gradlzero(\hat W_{H+1:1})) \matcol{U_l}{1}$,
which is a nonzero column vector,
and since $v_1 \in \nulsp(\hat W_{H:2})^\perp = \rowsp(\hat W_{H:2})$, $v_1^T (\hat W_{H:2})^T$ is a nonzero row vector.
From this observation, $\gradlzero(\hat W_{H+1:1}) v_0 v_1^T (\hat W_{H:2})^T$ is nonzero, and so is $\gradlzero(\hat W_{H+1:1}) V_1^T (\hat W_{H:2})^T$.

We are now ready to define the perturbation on $\hat W_{H+1}$:
\begin{equation*}
\Delta_{H+1} \defeq \frac{\epsilon \gradlzero(\hat W_{H+1:1}) V_1^T (\hat W_{H:2})^T}{\nfro{\gradlzero(\hat W_{H+1:1}) V_1^T (\hat W_{H:2})^T}},
\end{equation*}
so that $\hat W_{H+1} + \Delta_{H+1} \in B_\epsilon (\hat W_{H+1})$.
Then, observe that
\begin{align*}
\dotprod{\Delta_{H+1} \hat W_{H:2} V_1}{\gradlzero(\hat W_{H+1:1})}
= \dotprod{\Delta_{H+1}}{\gradlzero (\hat W_{H+1:1}) V_1^T (\hat W_{H:2})^T} > 0,
\end{align*}
by definition of $\Delta_{H+1}$. 
In other words, $\Delta_{H+1} \hat W_{H:2} V_1$ is an ascent direction of $\ell_0$ at $\hat W_{H+1:1}$.
Now choose the tuples
\begin{align*}
&(P_j)_{j=1}^{H+1} = (V_1, \hat W_2, \dots, \hat W_H, \hat W_{H+1} + \eta \Delta_{H+1}),\\
&(Q_j)_{j=1}^{H+1} = (V_1, \hat W_2, \dots, \hat W_H, \hat W_{H+1} - \eta \Delta_{H+1}),
\end{align*}
where $\eta \in (0,1]$ is chosen suitably. It is easy to verify that $(P_j)_{j=1}^{H+1} ,(Q_j)_{j=1}^{H+1} \in \prod_{j=1}^{H+1} \mc B_\epsilon(\hat W_j)$, and that the products
\begin{align*}
&P_{H+1:1} = \hat W_{H+1:1} + \eta \Delta_{H+1} \hat W_{H:2}V_1,\\
&Q_{H+1:1} = \hat W_{H+1:1} - \eta \Delta_{H+1} \hat W_{H:2}V_1.
\end{align*}
Since $\ell_0$ is differentiable, for small enough $\eta \in (0,1]$,
$\ell_0(P_{H+1:1}) > \ell_0(\hat W_{H+1:1}) > \ell_0(Q_{H+1:1})$,
proving \eqref{eq:thm1p1stmt}.
This construction is valid for any $\epsilon > 0$, so we are done.
\paragraph{Case 2: $\nulsp(\hat W_{H+1:2}) = \nulsp(\hat W_{H :2})$.}
By and large, the proof of this case goes the same, except that we need a little more care on what perturbations to make.
Define 
\begin{equation*}
j^* = \max \{ j \in [2,H] \mid \nulsp(\hat W_{j:2}) \supsetneq \nulsp(\hat W_{j-1:2}) \}.
\end{equation*}
When you start from $j = H$ down to $j = 2$ and compare $\nulsp(\hat W_{j:2})$ and $\nulsp(\hat W_{j-1:2})$,
the first iterate $j$ at which you have $\nulsp(\hat W_{j:2}) \neq \nulsp(\hat W_{j-1:2})$ is $j^*$.
If all null spaces of matrices from $\hat W_{H:2}$ to $\hat W_2$ are equal, $j^* = 2$ which follows from the notational convention that $\nulsp(\hat W_{1:2}) = \nulsp(I_{d_1}) = \{0 \}$.
According to $j^*$,
in Case 2 we perturb $\hat W_1$, $\hat W_{H+1}$, $\hat W_H$, $\dots$, $\hat W_{j^*}$ to get $(P_j)_{j=1}^{H+1}$ and $(Q_j)_{j=1}^{H+1}$.

Recall the definition of left-null space of matrix $A$: $\lnulsp(A) = \{ v \mid v^T A = 0\}$.
By definition of $j^*$, note that
\begin{align*}
&\nulsp(\hat W_{H+1:2}) = \nulsp(\hat W_{H:2}) = \cdots = \nulsp(\hat W_{j^*:2}) \\
\iff & \rowsp(\hat W_{H+1:2}) = \rowsp(\hat W_{H:2}) = \cdots = \rowsp(\hat W_{j^*:2})\\
\iff & \rank(\hat W_{H+1:2}) = \rank(\hat W_{H:2}) = \cdots = \rank(\hat W_{j^*:2}),
\end{align*}
which means the products are all rank-deficient (recall $\rank(\hat W_{H+1:2}) < d_y$ and all $d_j \geq d_y$),
and hence they all have nontrivial left-null spaces $\lnulsp(\hat W_{H:2}), \dots, \lnulsp(\hat W_{j^*:2})$ as well.

We choose some unit vectors as the following:
\begin{align*}
&v_0 = \matcol{U_r}{1},\\
&v_1 \in \nulsp(\hat W_{j^*:2}) \cap \nulsp(\hat W_{j^*-1:2})^\perp,\\
&v_{H+1} = \matcol{U_l}{1},\\
&v_{H} \in \lnulsp(\hat W_{H:2}),\\    
&\cdots\\
&v_{j^*} \in \lnulsp(\hat W_{j^*:2}).
\end{align*}
Then, for a $\gamma \in (0,\epsilon]$ whose value will be specified later, define
\begin{align*}
&\Delta_1 \defeq \gamma v_1 v_0^T \in \reals^{d_1 \times d_x},\\
&\Delta_{H+1} \defeq \gamma v_{H+1} v_{H}^T \in \reals^{d_y \times d_H},\\
&\cdots\\
&\Delta_{j^*+1} \defeq \gamma v_{j^*+1} v_{j^*}^T \in \reals^{d_{j^*+1} \times d_{j^*}},
\end{align*}
and $V_j \defeq \hat W_j + \Delta_j$ accordingly for $j = 1, j^*+1, \dots, H+1$.

By definition of $\Delta_j$'s, note that
\begin{align*}
&V_{H+1:j^*+1}\hat W_{j^*:2}V_1\\
=& V_{H+1:j^*+2}\hat W_{j^*+1:2}V_1 + V_{H+1:j^*+2}\Delta_{j^*+1} \hat W_{j^*:2}V_1
= V_{H+1:j^*+2}\hat W_{j^*+1:2}V_1 \numberthis \label{eq:can1dx}\\
=& V_{H+1:j^*+3}\hat W_{j^*+2:2}V_1 + V_{H+1:j^*+3}\Delta_{j^*+2} \hat W_{j^*+1:2}V_1
= V_{H+1:j^*+3}\hat W_{j^*+2:2}V_1 \numberthis \label{eq:can2dx}\\
=& \cdots\\
=& \hat W_{H+1:2}V_1 + \Delta_{H+1} \hat W_{H:2}V_1
= \hat W_{H+1:2}V_1 \numberthis \label{eq:can3dx}\\
=& \hat W_{H+1:1} + \hat W_{H+1:2}\Delta_1
= \hat W_{H+1:1}, \numberthis \label{eq:can4dx}
\end{align*}
where in \eqref{eq:can1dx} we used the definition that $v_{j^*} \in \lnulsp(\hat W_{j^*:2})$,
in \eqref{eq:can2dx} that $v_{j^*+1} \in \lnulsp(\hat W_{j^*+1:2})$,
in \eqref{eq:can3dx} that $v_{H} \in \lnulsp(\hat W_{H:2})$,
and in \eqref{eq:can4dx} that $v_1 \in \nulsp(\hat W_{j^*:2})$.

Now consider the following matrix product:
\begin{align*}
&(V_{H+1:j^*+1})^T \gradlzero(\hat W_{H+1:1}) V_1^T (\hat W_{j^*-1:2})^T\\
= &(\hat W_{j^*+1} + \Delta_{j^*+1})^T \cdots (\hat W_{H+1} + \Delta_{H+1})^T \gradlzero(\hat W_{H+1:1}) (\hat W_1 + \Delta_1)^T \hat W_2^T \cdots \hat W_{j^*-1}^T. \numberthis \label{eq:fullproddx}
\end{align*}
We are going to show that for small enough $\gamma \in (0,\epsilon]$, this product is nonzero.
If we expand \eqref{eq:fullproddx}, there are many terms in the summation. 
However, note that the expansion can be arranged in the following form:
\begin{align*}
&(\hat W_{j^*+1} + \Delta_{j^*+1})^T \cdots (\hat W_{H+1} + \Delta_{H+1})^T \gradlzero(\hat W_{H+1:1}) (\hat W_1 + \Delta_1)^T \hat W_2^T \cdots \hat W_{j^*-1}^T\\
=& C_0 + C_1 \gamma +  C_2 \gamma^2 + \cdots +  C_{H-j^*+2} \gamma^{H-j^*+2} \numberthis \label{eq:arrproddx}
\end{align*}
where $C_j \in \reals^{d_{j^*} \times d_{j^*-1}}$ for all $j$ and $C_j$ doesn't depend on $\gamma$, and specifically
\begin{align*}
&C_0 = \hat W_{j^*+1}^T \cdots \hat W_{H+1}^T \gradlzero(\hat W_{H+1:1}) \hat W_1^T \hat W_2^T \cdots \hat W_{j^*-1}^T,\\
&C_{H-j^*+2} = \frac{1}{\gamma^{H-j^*+2} } \Delta_{j^*+1}^T \cdots \Delta_{H+1}^T \gradlzero(\hat W_{H+1:1}) \Delta_1^T \hat W_2^T \cdots \hat W_{j^*-1}^T.
\end{align*}
Because the $C_0$ is exactly equal to $\frac{\partial \ell}{\partial W_{j^*}}$ evaluated at a critical point $( (\hat W_j)_{j=1}^{H+1} )$, $C_0 = 0$.
Also, due to definitions of $\Delta_j$'s,
\begin{align*}
C_{H-j^*+2}
=&(v_{j^*} v_{j^*+1}^T ) (v_{j^*+1} v_{j^*+2}^T )\cdots (v_{H} v_{H+1}^T ) \gradlzero(\hat W_{H+1:1}) (v_{0} v_{1}^T ) (\hat W_{j^*-1:2})^T\\
=&v_{j^*} v_{H+1}^T \gradlzero(\hat W_{H+1:1}) v_{0} v_{1}^T (\hat W_{j^*-1:2})^T.
\end{align*}
First, $v_{j^*}$ is a nonzero column vector.
Since $v_{H+1} =\matcol{U_l}{1}$ and $v_0 = \matcol{U_r}{1}$,
$v_{H+1}^T \gradlzero(\hat W_{H+1:1}) v_{0} = \sigmax(\gradlzero(\hat W_{H+1:1})) > 0$.
Also, since $v_1 \in \rowsp(\hat W_{j^*-1:2})$, $v_{1}^T (\hat W_{j^*-1:2})^T$ will be a nonzero row vector.
Thus, the product $C_{H-j^*+2}$ will be nonzero.

Since $C_{H-j^*+2} \neq 0$, we can pick any index $(\alpha, \beta)$ such that the $(\alpha,\beta)$-th entry of $C_{H-j^*+2}$, denoted as $[C_{H-j^*+2}]_{\alpha, \beta}$, is nonzero.
Then, the $(\alpha,\beta)$-th entry of \eqref{eq:arrproddx} can be written as
\begin{equation}
\label{eq:fullprodentrydx}
c_1 \gamma + c_2 \gamma^2 + \cdots +  c_{H-j^*+2} \gamma^{H-j^*+2},
\end{equation}
where $c_j = \matent{C_j}{\alpha}{\beta}$.
To show that the matrix product \eqref{eq:fullproddx} is nonzero, 
it suffices to show that its $(\alpha,\beta)$-th entry \eqref{eq:fullprodentrydx} is nonzero.
If $c_1 = \cdots = c_{H-j^*+1} = 0$, then with the choice of $\gamma = \epsilon$, \eqref{eq:fullprodentrydx} is trivially nonzero.
If some of $c_1, \dots, c_{H-j^*+1}$ are nonzero, we can scale $\gamma \in (0, \epsilon]$ arbitrarily small, so that
\begin{equation*}
|c_1 \gamma + \cdots + c_{H-j^*+1} \gamma^{H-j^*+1}| > |c_{H-j^*+2} \gamma^{H-j^*+2}|,
\end{equation*}
and thus \eqref{eq:fullprodentrydx} can never be zero. From this, with sufficiently small $\gamma$,
the matrix product \eqref{eq:fullproddx} is nonzero.

Now define the perturbation on $\hat W_{j^*}$:
\begin{equation*}
\Delta_{j^*} \defeq \frac{\epsilon (V_{H+1:j^*+1})^T \gradlzero(\hat W_{H+1:1}) V_1^T (\hat W_{j^*-1:2})^T}{\nfro{(V_{H+1:j^*+1})^T \gradlzero(\hat W_{H+1:1}) V_1^T (\hat W_{j^*-1:2})^T}},
\end{equation*}
so that $\hat W_{j^*} + \Delta_{j^*} \in B_\epsilon (\hat W_{j^*})$.
Then, observe that
\begin{align*}
&\dotprod{V_{H+1:j^*+1} \Delta_{j^*} \hat W_{j^*-1:2} V_1}{\gradlzero(\hat W_{H+1:1})}
= \tr((V_{H+1:j^*+1} \Delta_{j^*} \hat W_{j^*-1:2} V_1)^T\gradlzero(\hat W_{H+1:1}))\\
=& \tr(\Delta_{j^*}^T (V_{H+1:j^*+1})^T \gradlzero(\hat W_{H+1:1}) V_1^T (\hat W_{j^*-1:2})^T )
= \dotprod{\Delta_{j^*}}{(V_{H+1:j^*+1})^T \gradlzero(\hat W_{H+1:1}) V_1^T (\hat W_{j^*-1:2})^T } > 0.
\end{align*}
This means that $V_{H+1:j^*+1} \Delta_{j^*} \hat W_{j^*-1:2} V_1$ and $-V_{H+1:j^*+1} \Delta_{j^*} \hat W_{j^*-1:2} V_1$ are 
ascent and descent directions, respectively, of $\ell_0(R)$ at $\hat W_{H+1:1}$.
After that, the proof is very similar to the previous case. We can define
\begin{align*}
(P_j)_{j=1}^{H+1} &= (V_1, \hat W_2, \dots, \hat W_{j^*-1}, \hat W_{j^*} + \eta \Delta_{j^*}, V_{j^*+1}, \dots, V_{H+1}) \in \prod\nolimits_{j=1}^{H+1} \mc B_\epsilon(\hat W_j)\\
(Q_j)_{j=1}^{H+1} &= (V_1, \hat W_2, \dots, \hat W_{j^*-1}, \hat W_{j^*} - \eta \Delta_{j^*}, V_{j^*+1}, \dots, V_{H+1}) \in \prod\nolimits_{j=1}^{H+1} \mc B_\epsilon(\hat W_j),
\end{align*}
where $0<\eta\leq 1$ is small enough, to show that by differentiability of $\ell_0(R)$, we get $\ell((P_j)_{j=1}^{H+1}) > \ell((\hat W_j)_{j=1}^{H+1}) > \ell((Q_j)_{j=1}^{H+1})$.

\subsection{Proof of Part 1, if $d_y \geq d_x$}
\label{sec:thm1p1y}
First, note that $\gradlzero(\hat W_{H+1:1}) (\hat W_{H:1})^T = 0$,
because it is $\frac{\partial \ell}{\partial W_{H+1}}$ evaluated at a critical point $(\hat W_j)_{j=1}^{H+1}$.
This equation implies $\rowsp(\gradlzero(\hat W_{H+1:1}))^\perp \supseteq \rowsp(\hat W_{H:1})$.
Since $\gradlzero(\hat W_{H+1:1}) \neq 0$, 
$\rowsp(\gradlzero(\hat W_{H+1:1}))^\perp$ cannot be the whole $\reals^{d_x}$,
and it is a strict subspace of $\reals^{d_x}$.
Observe that $\hat W_{H:1} \in \reals^{d_H \times d_x}$ and $d_x \leq d_H$.
Since $\rowsp(\hat W_{H:1}) \subseteq \rowsp(\gradlzero(\hat W_{H+1:1}))^\perp \subsetneq \reals^{d_x}$,
this means $\rank(\hat W_{H:1}) < d_x$, hence $\lnulsp(\hat W_{H:1})$ is not a trivial subspace.

Now observe that
\begin{equation*}
\lnulsp(\hat W_{H:1}) \supseteq \lnulsp(\hat W_{H:2}) \supseteq \dots \supseteq \lnulsp(\hat W_H),
\end{equation*}
where some of left-null spaces in the right could be zero-dimensional.
The procedure of choosing the perturbation depends on these left-null spaces.
We can split the proof into two cases: $\lnulsp(\hat W_{H:1}) \neq \lnulsp(\hat W_{H:2})$ and $\lnulsp(\hat W_{H:1}) = \lnulsp(\hat W_{H :2})$.
Because the former case is simpler, we prove the former case first.

Before we dive in, again take SVD of $\gradlzero(\hat W_{H+1:1}) = U_l \Sigma U_r^T$.
Since $\gradlzero(\hat W_{H+1:1}) \neq 0$, there is at least one positive singular value, so $\sigmax(\gradlzero(\hat W_{H+1:1})) > 0$. 
Recall the notation that $\matcol{U_l}{1}$ and $\matcol{U_r}{1}$ are first column vectors of $U_l$ and $U_r$, respectively.

\paragraph{Case 1: $\lnulsp(\hat W_{H:1}) \neq \lnulsp(\hat W_{H:2})$.}
In this case, $\lnulsp(\hat W_{H:1}) \supsetneq \lnulsp(\hat W_{H:2})$.
We will perturb $\hat W_1$ and $\hat W_{H+1}$ to obtain the desired tuples $(P_j)_{j=1}^{H+1}$ and $(Q_j)_{j=1}^{H+1}$.

Now choose two unit vectors $v_H$ and $v_{H+1}$, as the following:
\begin{equation*}
v_H \in \lnulsp(\hat W_{H:1}) \cap \lnulsp(\hat W_{H:2})^\perp,~
v_{H+1} = \matcol{U_l}{1},
\end{equation*}
and then define $\Delta_{H+1} \defeq \epsilon v_{H+1} v_H^T \in \reals^{d_y \times d_H}$, and $V_{H+1} \defeq \hat W_{H+1}+\Delta_{H+1}$.
We can check $V_{H+1} \in \mc B_\epsilon(\hat W_{H+1})$ from the fact that $v_H$ and $v_{H+1}$ are unit vectors.
Since $v_H \in \lnulsp(\hat W_{H:1})$, observe that
\begin{equation*}
V_{H+1} \hat W_{H:1}= \hat W_{H+1:1} + \epsilon v_{H+1} v_H^T \hat W_{H:1} = \hat W_{H+1:1}.
\end{equation*}
With this definition of $V_{H+1}$, we can also see that
\begin{align*}
(\hat W_{H:2})^T V_{H+1}^T \gradlzero(\hat W_{H+1:1})
= (\hat W_{H+1:2})^T \gradlzero(\hat W_{H+1:1}) + \epsilon (\hat W_{H:2})^T v_H v_{H+1}^T \gradlzero(\hat W_{H+1:1}).
\end{align*}
Note that $(\hat W_{H+1:2})^T \gradlzero(\hat W_{H+1:1})$ is exactly equal to $\frac{\partial \ell}{\partial W_{1}}$ evaluated at $(\hat W_j)_{j=1}^{H+1}$, hence is zero by assumption that $(\hat W_j)_{j=1}^{H+1}$ is a critical point.
Since $v_H \in \lnulsp(\hat W_{H:2})^\perp = \colsp(\hat W_{H:2})$, $(\hat W_{H:2})^T v_H$ is a nonzero column vector,
and since $v_{H+1}  = \matcol{U_l}{1}$, $v_{H+1}^T \gradlzero(\hat W_{H+1:1}) = \sigmax(\gradlzero(\hat W_{H+1:1})) (\matcol{U_r}{1})^T$, which is a nonzero row vector.
From this observation, we can see that $(\hat W_{H:2})^T v_H v_{H+1}^T \gradlzero(\hat W_{H+1:1})$ is nonzero, and so is $(\hat W_{H:2})^T V_{H+1}^T \gradlzero(\hat W_{H+1:1})$.

Now define the perturbation on $\hat W_1$:
\begin{equation*}
\Delta_1 \defeq \frac{\epsilon (\hat W_{H:2})^T V_{H+1}^T \gradlzero(\hat W_{H+1:1})}{\nfro{(\hat W_{H:2})^T V_{H+1}^T \gradlzero(\hat W_{H+1:1})}},
\end{equation*}
so that $\hat W_1 + \Delta_1 \in B_\epsilon (\hat W_1)$.
Then, observe that
\begin{align*}
&\dotprod{V_{H+1} \hat W_{H:2} \Delta_1}{\gradlzero(\hat W_{H+1:1})}
= \tr((V_{H+1} \hat W_{H:2} \Delta_1)^T \gradlzero (\hat W_{H+1:1}))\\
=& \tr(\Delta_1^T (\hat W_{H:2})^T V_{H+1}^T \gradlzero(\hat W_{H+1:1}))
= \dotprod{\Delta_1}{(\hat W_{H:2})^T V_{H+1}^T \gradlzero(\hat W_{H+1:1})} > 0,
\end{align*}
by definition of $\Delta_1$. 
This means that $V_{H+1} \hat W_{H:2} \Delta_1$ and $-V_{H+1} \hat W_{H:2} \Delta_1$ are 
ascent and descent directions, respectively, of $\ell_0(R)$ at $\hat W_{H+1:1}$.
Since $\ell_0$ is a differentiable function, there exists small enough $0< \eta \leq 1$ that satisfies
\begin{align*}
&\ell_0(\hat W_{H+1:1}+\eta V_{H+1} \hat W_{H:2} \Delta_1) > \ell_0(\hat W_{H+1:1}),\\
&\ell_0(\hat W_{H+1:1}-\eta V_{H+1} \hat W_{H:2} \Delta_1) < \ell_0(\hat W_{H+1:1}).
\end{align*}

Now define
\begin{align*}
&(P_j)_{j=1}^{H+1} = (\hat W_1 + \eta \Delta_1, \hat W_2, \dots, \hat W_H, V_{H+1}),\\
&(Q_j)_{j=1}^{H+1} = (\hat W_1 - \eta \Delta_1, \hat W_2, \dots, \hat W_H, V_{H+1}).
\end{align*}
We can check $(P_j)_{j=1}^{H+1} ,(Q_j)_{j=1}^{H+1} \in \prod_{j=1}^{H+1} \mc B_\epsilon(\hat W_j)$, and
\begin{align*}
&P_{H+1:1} = \hat W_{H+1:1} + \eta V_{H+1} \hat W_{H:2} \Delta_1.\\
&Q_{H+1:1} = \hat W_{H+1:1} - \eta V_{H+1} \hat W_{H:2} \Delta_1.
\end{align*}
By definition of $\ell( (W_j)_{j=1}^{H+1} )$, this shows that $\ell( (P_j)_{j=1}^{H+1} ) > \ell( (\hat W_j)_{j=1}^{H+1} ) > \ell( (Q_j)_{j=1}^{H+1} )$. 
This construction holds for any $\epsilon > 0$, proving that $(\hat W_j)_{j=1}^{H+1} $ can be neither a local maximum nor a local minimum.

\paragraph{Case 2: $\lnulsp(\hat W_{H:1}) = \lnulsp(\hat W_{H:2})$.}
By and large, the proof of this case goes the same, except that we need a little more care on what perturbations to make.
Define 
\begin{equation*}
j^* = \min \{ j \in [2,H] \mid \lnulsp(\hat W_{H:j}) \supsetneq \lnulsp(\hat W_{H:j+1}) \}.
\end{equation*}
When you start from $j = 2$ up to $j = H$ and compare $\lnulsp(\hat W_{H:j})$ and $\lnulsp(\hat W_{H:j+1})$,
the first iterate $j$ at which you have $\lnulsp(\hat W_{H:j}) \neq \lnulsp(\hat W_{H:j+1})$ is $j^*$.
If all left-null spaces of matrices from $\hat W_{H:2}$ to $\hat W_H$ are equal, $j^* = H$ which follows from the notational convention that $\lnulsp(\hat W_{H:H+1}) = \lnulsp(I_{d_H}) = \{0 \}$.
According to $j^*$,
in Case 2 we perturb $\hat W_{H+1}$, $\hat W_1$, $\hat W_2$, $\dots$, $\hat W_{j^*}$ to get $(P_j)_{j=1}^{H+1}$ and $(Q_j)_{j=1}^{H+1}$.

By definition of $j^*$, note that
\begin{align*}
&\lnulsp(\hat W_{H:1}) = \lnulsp(\hat W_{H:2}) = \cdots = \lnulsp(\hat W_{H:j^*}) \\
\iff &\colsp(\hat W_{H:1}) = \colsp(\hat W_{H:2}) = \cdots = \colsp(\hat W_{H:j^*}) \\
\iff &\rank(\hat W_{H:1}) = \rank(\hat W_{H:2}) = \cdots = \rank(\hat W_{H:j^*})
\end{align*}
which means the products are all rank-deficient (recall $\rank(\hat W_{H:1}) < d_x$ and all $d_j \geq d_x$),
and hence they all have nontrivial null spaces $\nulsp(\hat W_{H:2}), \dots, \nulsp(\hat W_{H:j^*})$ as well.

We choose some unit vectors as the following:
\begin{align*}
&v_0 = \matcol{U_r}{1},\\
&v_1 \in \nulsp(\hat W_{H:2}),\\
&\cdots\\
&v_{j^*-1} \in \nulsp(\hat W_{H:j^*})\\
&v_H \in \lnulsp(\hat W_{H:j^*}) \cap \lnulsp(\hat W_{H:j^*+1})^\perp,\\
&v_{H+1} = \matcol{U_l}{1}.
\end{align*}
Then, for a $\gamma \in (0,\epsilon]$ whose value will be specified later, define
\begin{align*}
&\Delta_1 \defeq \gamma v_1 v_0^T \in \reals^{d_1 \times d_x},\\
&\cdots\\
&\Delta_{j^*-1} \defeq \gamma v_{j^*-1} v_{j^*-2}^T \in \reals^{d_{j^*-1} \times d_{j^*-2}},\\
&\Delta_{H+1} \defeq \gamma v_{H+1} v_{H}^T \in \reals^{d_y \times d_H},
\end{align*}
and $V_j \defeq \hat W_j + \Delta_j$ accordingly for $j = 1, \dots, j^*-1, H+1$.

By definition of $\Delta_j$'s, note that
\begin{align*}
& V_{H+1} \hat W_{H:j^*} V_{j^*-1:1}\\
=& V_{H+1} \hat W_{H:j^*-1} V_{j^*-2:1} + V_{H+1} \hat W_{H:j^*} \Delta_{j^*-1} V_{j^*-2:1}
= V_{H+1} \hat W_{H:j^*-1} V_{j^*-2:1} \numberthis \label{eq:can1dy}\\
=& V_{H+1} \hat W_{H:j^*-2} V_{j^*-3:1} + V_{H+1} \hat W_{H:j^*-1} \Delta_{j^*-2} V_{j^*-3:1}
= V_{H+1} \hat W_{H:j^*-2} V_{j^*-3:1} \numberthis \label{eq:can2dy}\\
=& \cdots\\
=& V_{H+1} \hat W_{H:1} + V_{H+1} \hat W_{H:2} \Delta_1
= V_{H+1} \hat W_{H:1} \numberthis \label{eq:can3dy}\\
=& \hat W_{H+1:1} + \Delta_{H+1} \hat W_{H:1}
= \hat W_{H+1:1}, \numberthis \label{eq:can4dy}
\end{align*}
where in \eqref{eq:can1dy} we used the definition that $v_{j^*-1} \in \nulsp(\hat W_{H:j^*})$,
in \eqref{eq:can2dy} that $v_{j^*-2} \in \nulsp(\hat W_{H:j^*-1})$,
in \eqref{eq:can3dy} that $v_{1} \in \nulsp(\hat W_{H:2})$,
and in \eqref{eq:can4dy} that $v_H \in \lnulsp(\hat W_{H:j^*})$.

Now consider the following matrix product:
\begin{align*}
&(\hat W_{H:j^*+1})^T V_{H+1}^T \gradlzero(\hat W_{H+1:1}) (V_{j^*-1:1})^T\\
&= (\hat W_{H:j^*+1})^T (\hat W_{H+1} + \Delta_{H+1})^T \gradlzero(\hat W_{H+1:1}) (\hat W_1 + \Delta_1)^T \cdots (\hat W_{j^*-1} + \Delta_{j^*-1})^T.\numberthis \label{eq:fullproddy}
\end{align*}
We are going to show that for small enough $\gamma \in (0,\epsilon]$, this product is nonzero.
If we expand \eqref{eq:fullproddy}, there are many terms in the summation. 
However, note that the expansion can be arranged in the following form:
\begin{align*}
&(\hat W_{H:j^*+1})^T (\hat W_{H+1} + \Delta_{H+1})^T \gradlzero(\hat W_{H+1:1}) (\hat W_1 + \Delta_1)^T \cdots (\hat W_{j^*-1} + \Delta_{j^*-1})^T\\
=& C_0 + C_1 \gamma +  C_2 \gamma^2 + \cdots +  C_{j^*} \gamma^{j^*} \numberthis \label{eq:arrproddy}
\end{align*}
where $C_j \in \reals^{d_{j^*} \times d_{j^*-1}}$ for all $j$ and $C_j$ doesn't depend on $\gamma$, and specifically
\begin{align*}
&C_0 = \hat W_{j^*+1}^T \cdots \hat W_{H+1}^T \gradlzero(\hat W_{H+1:1}) \hat W_1^T \hat W_2^T \cdots \hat W_{j^*-1}^T,\\
&C_{j^*} = \frac{1}{\gamma^{j^*} } \hat W_{j^*+1}^T \cdots \hat W_{H}^T \Delta_{H+1}^T \gradlzero(\hat W_{H+1:1}) \Delta_1^T \cdots \Delta_{j^*-1}^T.
\end{align*}
Because the $C_0$ is exactly equal to $\frac{\partial \ell}{\partial W_{j^*}}$ evaluated at a critical point $( (\hat W_j)_{j=1}^{H+1} )$, $C_0 = 0$.
Also, due to definitions of $\Delta_j$'s,
\begin{align*}
C_{j^*}
=&(\hat W_{H:j^*+1})^T (v_{H} v_{H+1}^T) \gradlzero(\hat W_{H+1:1}) (v_0v_1^T) (v_1 v_2^T) \cdots (v_{j^*-2} v_{j^*-1}^T)\\
=&(\hat W_{H:j^*+1})^T v_{H} v_{H+1}^T \gradlzero(\hat W_{H+1:1}) v_{0} v_{j^*-1}^T.
\end{align*}
First, since $v_H \in \colsp(\hat W_{H:j^*+1})$, $(\hat W_{H:j^*+1})^T v_{H}$ is a nonzero column vector.
Also, since $v_{H+1} =\matcol{U_l}{1}$ and $v_0 = \matcol{U_r}{1}$,
the product $v_{H+1}^T \gradlzero(\hat W_{H+1:1}) v_{0} = \sigmax(\gradlzero(\hat W_{H+1:1})) > 0$.
Finally, $v_{j^*-1}^T$ is a nonzero row vector. Thus, the product $C_{j^*}$ will be nonzero.

Since $C_{j^*} \neq 0$, we can pick any index $(\alpha, \beta)$ such that the $(\alpha,\beta)$-th entry of $C_{j^*}$, denoted as $[C_{j^*}]_{\alpha, \beta}$, is nonzero.
Then, the $(\alpha,\beta)$-th entry of \eqref{eq:arrproddy} can be written as
\begin{equation}
\label{eq:fullprodentrydy}
c_1 \gamma + c_2 \gamma^2 + \cdots +  c_{j^*} \gamma^{j^*},
\end{equation}
where $c_j = [C_j]_{\alpha,\beta}$.
To show that the matrix product \eqref{eq:fullproddy} is nonzero, 
it suffices to show that its $(\alpha,\beta)$-th entry \eqref{eq:fullprodentrydy} is nonzero.
If $c_1 = \cdots = c_{j^*-1} = 0$, then with the choice of $\gamma = \epsilon$, \eqref{eq:fullprodentrydy} is trivially nonzero.
If some of $c_1, \dots, c_{j^*-1}$ are nonzero, we can scale $\gamma \in (0, \epsilon]$ arbitrarily small, so that
\begin{equation*}
|c_1 \gamma + \cdots + c_{j^*-1} \gamma^{j^*-1}| > |c_{j^*} \gamma^{j^*}|,
\end{equation*}
and thus \eqref{eq:fullprodentrydy} can never be zero. From this, with sufficiently small $\gamma$,
the matrix product \eqref{eq:fullproddy} is nonzero.

Now define the perturbation on $\hat W_{j^*}$:
\begin{equation*}
\Delta_{j^*} \defeq \frac{\epsilon (\hat W_{H:j^*+1})^T V_{H+1}^T \gradlzero(\hat W_{H+1:1}) (V_{j^*-1:1})^T}{\nfro{(\hat W_{H:j^*+1})^T V_{H+1}^T \gradlzero(\hat W_{H+1:1}) (V_{j^*-1:1})^T}},
\end{equation*}
so that $\hat W_{j^*} + \Delta_{j^*} \in B_\epsilon (\hat W_{j^*})$.
Then, observe that
\begin{align*}
&\dotprod{V_{H+1} \hat W_{H:j^*+1} \Delta_{j^*} V_{j^*-1:1}}{\gradlzero(\hat W_{H+1:1})}
= \tr((V_{H+1} \hat W_{H:j^*+1} \Delta_{j^*} V_{j^*-1:1})^T\gradlzero(\hat W_{H+1:1}))\\
=& \tr(\Delta_{j^*}^T (\hat W_{H:j^*+1})^T V_{H+1}^T \gradlzero(\hat W_{H+1:1}) (V_{j^*-1:1})^T )
= \dotprod{\Delta_{j^*}}{(\hat W_{H:j^*+1})^T V_{H+1}^T \gradlzero(\hat W_{H+1:1}) (V_{j^*-1:1})^T } > 0.
\end{align*}
This means that $V_{H+1} \hat W_{H:j^*+1} \Delta_{j^*} V_{j^*-1:1}$ and $-V_{H+1} \hat W_{H:j^*+1} \Delta_{j^*} V_{j^*-1:1}$ are 
ascent and descent directions, respectively, of $\ell_0(R)$ at $\hat W_{H+1:1}$.
After that, the proof is very similar to the previous case. We can define
\begin{align*}
(P_j)_{j=1}^{H+1} &= (V_1, \dots, V_{j^*-1}, \hat W_{j^*} + \eta \Delta_{j^*}, \hat W_{j^*+1}, \dots, \hat W_H, V_{H+1}) \in \prod\nolimits_{j=1}^{H+1} \mc B_\epsilon(\hat W_j)\\
(Q_j)_{j=1}^{H+1} &= (V_1, \dots, V_{j^*-1}, \hat W_{j^*} - \eta \Delta_{j^*}, \hat W_{j^*+1}, \dots, \hat W_H, V_{H+1}) \in \prod\nolimits_{j=1}^{H+1} \mc B_\epsilon(\hat W_j),
\end{align*}
where $0<\eta\leq 1$ is small enough, to show that by differentiability of $\ell_0(R)$, we get $\ell((P_j)_{j=1}^{H+1}) > \ell((\hat W_j)_{j=1}^{H+1}) > \ell((Q_j)_{j=1}^{H+1})$.

\subsection{Proof of Part 2(a)}
In this part, we show that if $\gradlzero(\hat W_{H+1:1}) = 0$ and 
$\hat W_{H+1:1}$ is a local min of $\ell_0$, then $(\hat W_j)_{j=1}^{H+1}$ is a local min of $\ell$.
The proof for local max case can be done in a very similar way.

Since $\hat W_{H+1:1}$ is a local minimum of $\ell_0$, there exists $\epsilon > 0$ such that,
for any $R$ satisfying $\nfro{R-\hat W_{H+1:1}} \leq \epsilon$, we have $\ell_0(R) \geq \ell_0(\hat W_{H+1:1})$.
We prove that $(\hat W_j)_{j=1}^{H+1}$ is a local minimum of $\ell$ by showing that 
there exists a neighborhood of $(\hat W_j)_{j=1}^{H+1}$ in which
any point $(V_j)_{j=1}^{H+1}$ satisfies $\ell( (V_j)_{j=1}^{H+1} ) \geq \ell( (\hat W_j)_{j=1}^{H+1})$.

Now define
\begin{equation*}
0<
\epsilon_j
\leq
\frac{\epsilon}{2(H+1) \max\big \{\nfro{\hat W_{H+1:j+1}} \nfro{\hat W_{j-1:1}},1 \big \}}.
\end{equation*}
Observe that $\frac{a}{\max \{a,1\}} \leq 1$ for $a \geq 0$.
Then, for all $j \in [H+1]$, pick any $V_j$ such that $\nfro{V_j - \hat W_j} \leq \epsilon_j$.
Denote $\Delta_j = V_j-\hat W_j$ for all $j$.
Now, by triangle inequality and submultiplicativity of Frobenius norm,
\begin{align*}
\nfro{(\hat W_{H+1} + \Delta_{H+1}) \cdots (\hat W_1 + \Delta_1) - \hat W_{H+1:1}}
\leq &\sum_{j=1}^{H+1} \nfro{ \hat W_{H+1:j+1} \Delta_j \hat W_{j-1:1}} + O(\max_j \nfro{\Delta_j}^2)\\
\leq &\sum_{j=1}^{H+1} \nfro{ \hat W_{H+1:j+1} } \nfro { \Delta_j } \nfro{ \hat W_{j-1:1} }  + O(\max_j \epsilon_j^2)\\
\leq &\frac{\epsilon}{2} + O(\max_j \epsilon_j^2) \leq \epsilon,
\end{align*}
for small enough $\epsilon_j$'s.

Given this, for any $(V_j)_{j=1}^{H+1}$ in the neighborhood of $(\hat W_j)_{j=1}^{H+1}$ defined by $\epsilon_j$'s, $\nfro{V_{H+1:1} - \hat W_{H+1:1}} \leq \epsilon$, so $\ell_0(V_{H+1:1}) \geq \ell_0(\hat W_{H+1:1})$, meaning $\ell((V_j)_{j=1}^{H+1}) \geq \ell((\hat W_j)_{j=1}^{H+1})$.
Thus, $(\hat W_j)_{j=1}^{H+1}$ is a local minimum of $\ell$.

\subsection{Proof of Part 2(b)}
For this part, we want to show that if $\gradlzero(\hat W_{H+1:1}) = 0$, then
$(\hat W_j)_{j=1}^{H+1}$ is a global min (or max) of $\ell$ if and only if
$\hat W_{H+1:1}$ is a global min (or max) of $\ell_0$.
We prove this by showing the following:
if $d_j \geq \min\{ d_x,d_y \}$ for all $j \in [H]$, for any $R \in \reals^{d_y \times d_x}$ there exists a decomposition $(W_j)_{j=1}^{H+1}$ such that $R = W_{H+1:1}$.

We divide the proof into two cases: $d_x \geq d_y$ and $d_y \geq d_x$.
\paragraph{Case 1: $d_x \geq d_y$.}
If $d_x \geq d_y$, by assumption $d_j \geq d_y$ for all $j \in [H]$.
Recall that $W_1 \in \reals^{d_1 \times d_x}$. 
Given $R \in \reals^{d_y \times d_x}$, we can fill the first $d_y$ rows of $W_1$ with $R$ and let any other entries be zero.
For all the other matrices $W_2, \dots, W_{H+1}$, we put ones to the diagonal entries while putting zeros to all the other entries.
We can check that, by this construction, $R = W_{H+1:1}$ for this given $R$.
\paragraph{Case 2: $d_y \geq d_x$.}
If $d_y \geq d_x$, we have $d_j \geq d_x$ for all $j \in [H]$.
Recall $W_{H+1} \in \reals^{d_y \times d_H}$.
Given $R \in \reals^{d_y \times d_x}$, we can fill the first $d_x$ columns of $W_{H+1}$ with $R$ and let any other entries be zero.
For all the other matrices $W_1, \dots, W_H$, we put ones to the diagonal entries while putting zeros to all the other entries.
By this construction, $R = W_{H+1:1}$ for given $R$.

Once this fact is given,
by $\ell((W_j)_{j=1}^{H+1}) = \ell_0(W_{H+1:1})$,
\begin{equation*}
\inf_R \ell_0(R) = \!\! \inf_{W_{H+1:1}} \!\! \ell_0(W_{H+1:1}) =\!\! \inf_{(W_j)_{j=1}^{H+1}} \!\! \ell((W_j)_{j=1}^{H+1}).
\end{equation*}
Thus, any $(\hat W_j)_{j=1}^{H+1}$ attaining a global min of $\ell$ must have $\inf_R \ell_0(R) = \ell_0(\hat W_{H+1:1})$, so $\hat W_{H+1:1}$ is also a global min of $\ell_0(R)$.
Conversely, if $\ell_0(\hat W_{H+1:1}) = \inf \ell_0 (R)$, then
$\ell((\hat W_j)_{j=1}^{H+1}) = \inf \ell((W_j)_{j=1}^{H+1})$,
so $(\hat W_j)_{j=1}^{H+1}$ is a global min of $\ell$. We can prove the global max case similarly.

\subsection{Proof of Part 3 and 3(a)}
Suppose there exists $j^* \in [H+1]$ such that $\hat W_{H+1:j^*+1}$ has full row rank and $\hat W_{j^*-1:1}$ has full column rank.
For simplicity, define $A \defeq \hat W_{H+1:j^*+1}$ and $B \defeq \hat W_{j^*-1:1}$. 
Since $A^T$ has linearly independent columns, $B^T$ has linearly independent rows, and 
$\partial \ell/\partial W_{j^*} = 0$ at $(\hat W_j)_{j=1}^{H+1}$,
$A^T \gradlzero(\hat W_{H+1:1}) B^T = 0 \implies \gradlzero(\hat W_{H+1:1}) = 0$,
hence Parts 2(a) and 2(b) are implied.

For Part 3(a), we want to prove that if $(\hat W_j)_{j=1}^{H+1}$ is a local min of $\ell$, then $\hat W_{H+1:1}$ is a local min of $\ell_0$.
By definition of local min, $\exists \epsilon > 0$ such that,
for any $(V_j)_{j=1}^{H+1}$ for which $\nfro{V_j - \hat W_j} \leq \epsilon$ (for $j \in [H+1]$), we have $\ell( (V_j)_{j=1}^{H+1} ) \geq \ell( (\hat W_j)_{j=1}^{H+1} )$.
To show that $\hat W_{H+1:1}$ is a local min of $\ell_0$, 
we have to show there exists a neighborhood of $\hat W_{H+1:1}$ such that,
any point $R$ in that neighborhood satisfies $\ell_0(R) \geq \ell_0(\hat W_{H+1:1})$.
To prove this, we state the following lemma:
\begin{lemma}
	\label{lem:thm1p3atechlem}
	Suppose $A \defeq \hat W_{H+1:j^*+1}$ has full row rank and $B \defeq \hat W_{j^*-1:1}$ has full column rank. 
	Then, any $R$ satisfying $\nfro{R-\hat W_{H+1:1}} \leq \sigmin(A) \sigmin(B) \epsilon$ can be decomposed into $R = V_{H+1:1}$, where
	\begin{align*}
	V_{j^*} = \hat W_{j^*} + A^T (AA^T)^{-1} (R-\hat W_{H+1:1}) (B^T B)^{-1} B^T,
	\end{align*}
	and $V_j = \hat W_j$ for $j \neq j^*$. Also, $\nfro{V_j - \hat W_j} \leq \epsilon$ for all $j$.
\end{lemma}
\begin{proof}
Since $A \defeq \hat W_{H+1:j^*+1}$ has full row rank and $B \defeq \hat W_{j^*-1:1}$ has full column rank,
$\sigmin(A) > 0$, $\sigmin(B) > 0$, and $AA^T$ and $B^TB$ are invertible.
Consider any $R$ satisfying $\nfro{R-\hat W_{H+1:1}} \leq \sigmin(A) \sigmin(B) \epsilon$.
Given the definitions of $V_j$'s in the statement of the lemma,
we can check the identity that $R = V_{H+1:1}$ by
\begin{align*}
V_{H+1:1} = A V_j B = A \hat W_j B + (R-\hat W_{H+1:1})
= \hat W_{H+1:1} + (R-\hat W_{H+1:1}) = R.
\end{align*}

Now It is left to show that $\nfro{V_{j^*}-\hat W_{j^*}} \leq \epsilon$, so that $(V_j)_{j=1}^{H+1}$ indeed satisfies $\nfro{V_j - \hat W_j} \leq \epsilon$ for all $j$.
We can show that
\begin{align*}
\sigmax(A^T ( A A^T )^{-1}) = 1/\sigmin(A),~
\sigmax(( B^T B )^{-1} B^T ) = 1/\sigmin(B).
\end{align*}
Therefore,
\begin{align*}
\nfro{V_{j^*} - \hat W_{j^*}} 
=& \nfro{A^T (AA^T)^{-1} (R-\hat W_{H+1:1}) (B^T B)^{-1} B^T} \\
\leq& \sigmax(A^T (AA^T)^{-1}) \sigmax((B^T B)^{-1} B^T) \nfro{R-\hat W_{H+1:1}}\\
\leq& \frac{1}{\sigmin(A) \sigmin(B)} \cdot \sigmin(A) \sigmin(B) \epsilon = \epsilon.
\end{align*}
\end{proof}
The lemma shows that for any $R = V_{H+1:1}$ satisfying $\nfro{R-\hat W_{H+1:1}} \leq \sigmin(A) \sigmin(B) \epsilon$, we have $\ell_0(R) = \ell_0(V_{H+1:1})= \ell( (V_j)_{j=1}^{H+1} ) \geq \ell( (\hat W_j)_{j=1}^{H+1} ) = \ell_0(\hat W_{H+1:1})$.
We can prove the local maximum part by a similar argument.

%% file: main_critview.bbl
\begin{thebibliography}{42}
\providecommand{\natexlab}[1]{#1}
\providecommand{\url}[1]{\texttt{#1}}
\expandafter\ifx\csname urlstyle\endcsname\relax
  \providecommand{\doi}[1]{doi: #1}\else
  \providecommand{\doi}{doi: \begingroup \urlstyle{rm}\Url}\fi

\bibitem[Allen-Zhu et~al.(2018)Allen-Zhu, Li, and Song]{allen2018convergence}
Zeyuan Allen-Zhu, Yuanzhi Li, and Zhao Song.
\newblock A convergence theory for deep learning via over-parameterization.
\newblock \emph{arXiv preprint arXiv:1811.03962}, 2018.

\bibitem[Baldi \& Hornik(1989)Baldi and Hornik]{baldi1989neural}
Pierre Baldi and Kurt Hornik.
\newblock Neural networks and principal component analysis: Learning from
  examples without local minima.
\newblock \emph{Neural networks}, 2\penalty0 (1):\penalty0 53--58, 1989.

\bibitem[Bengio et~al.(2006)Bengio, Roux, Vincent, Delalleau, and
  Marcotte]{bengio2006convex}
Yoshua Bengio, Nicolas~L Roux, Pascal Vincent, Olivier Delalleau, and Patrice
  Marcotte.
\newblock Convex neural networks.
\newblock In \emph{Advances in neural information processing systems}, pp.\
  123--130, 2006.

\bibitem[Brutzkus \& Globerson(2017)Brutzkus and
  Globerson]{brutzkus2017globally}
Alon Brutzkus and Amir Globerson.
\newblock Globally optimal gradient descent for a {ConvNet} with gaussian
  inputs.
\newblock In \emph{International Conference on Machine Learning}, pp.\
  605--614, 2017.

\bibitem[Brutzkus et~al.(2018)Brutzkus, Globerson, Malach, and
  Shalev-Shwartz]{brutzkus2018sgd}
Alon Brutzkus, Amir Globerson, Eran Malach, and Shai Shalev-Shwartz.
\newblock {SGD} learns over-parameterized networks that provably generalize on
  linearly separable data.
\newblock In \emph{International Conference on Learning Representations}, 2018.

\bibitem[Choromanska et~al.(2015)Choromanska, Henaff, Mathieu, Arous, and
  LeCun]{choromanska2015loss}
Anna Choromanska, Mikael Henaff, Michael Mathieu, G{\'e}rard~Ben Arous, and
  Yann LeCun.
\newblock The loss surfaces of multilayer networks.
\newblock In \emph{Artificial Intelligence and Statistics}, pp.\  192--204,
  2015.

\bibitem[Clevert et~al.(2015)Clevert, Unterthiner, and
  Hochreiter]{clevert2015fast}
Djork-Arn{\'e} Clevert, Thomas Unterthiner, and Sepp Hochreiter.
\newblock Fast and accurate deep network learning by exponential linear units
  ({ELU}s).
\newblock \emph{arXiv preprint arXiv:1511.07289}, 2015.

\bibitem[Du et~al.(2018{\natexlab{a}})Du, Lee, Li, Wang, and
  Zhai]{du2018gradientB}
Simon~S Du, Jason~D Lee, Haochuan Li, Liwei Wang, and Xiyu Zhai.
\newblock Gradient descent finds global minima of deep neural networks.
\newblock \emph{arXiv preprint arXiv:1811.03804}, 2018{\natexlab{a}}.

\bibitem[Du et~al.(2018{\natexlab{b}})Du, Lee, Tian, Singh, and
  Poczos]{du2017gradient}
Simon~S Du, Jason~D Lee, Yuandong Tian, Aarti Singh, and Barnabas Poczos.
\newblock Gradient descent learns one-hidden-layer {CNN}: Don’t be afraid of
  spurious local minima.
\newblock In \emph{International Conference on Machine Learning}, pp.\
  1338--1347, 2018{\natexlab{b}}.

\bibitem[Du et~al.(2018{\natexlab{c}})Du, Zhai, Poczos, and
  Singh]{du2018gradientA}
Simon~S Du, Xiyu Zhai, Barnabas Poczos, and Aarti Singh.
\newblock Gradient descent provably optimizes over-parameterized neural
  networks.
\newblock \emph{arXiv preprint arXiv:1810.02054}, 2018{\natexlab{c}}.

\bibitem[Feizi et~al.(2017)Feizi, Javadi, Zhang, and Tse]{feizi2017porcupine}
Soheil Feizi, Hamid Javadi, Jesse Zhang, and David Tse.
\newblock Porcupine neural networks:(almost) all local optima are global.
\newblock \emph{arXiv preprint arXiv:1710.02196}, 2017.

\bibitem[Freeman \& Bruna(2017)Freeman and Bruna]{freeman2017topology}
C~Daniel Freeman and Joan Bruna.
\newblock Topology and geometry of half-rectified network optimization.
\newblock In \emph{International Conference on Learning Representations}, 2017.

\bibitem[Haeffele \& Vidal(2017)Haeffele and Vidal]{haeffele2017global}
Benjamin~D Haeffele and Ren{\'e} Vidal.
\newblock Global optimality in neural network training.
\newblock In \emph{Proceedings of the IEEE Conference on Computer Vision and
  Pattern Recognition}, pp.\  7331--7339, 2017.

\bibitem[Kawaguchi(2016)]{kawaguchi2016deep}
Kenji Kawaguchi.
\newblock Deep learning without poor local minima.
\newblock In \emph{Advances in Neural Information Processing Systems}, pp.\
  586--594, 2016.

\bibitem[Klambauer et~al.(2017)Klambauer, Unterthiner, Mayr, and
  Hochreiter]{klambauer2017self}
G{\"u}nter Klambauer, Thomas Unterthiner, Andreas Mayr, and Sepp Hochreiter.
\newblock Self-normalizing neural networks.
\newblock In \emph{Advances in Neural Information Processing Systems}, pp.\
  972--981, 2017.

\bibitem[Krantz \& Parks(2002)Krantz and Parks]{krantz2002primer}
Steven~G Krantz and Harold~R Parks.
\newblock \emph{A primer of real analytic functions}.
\newblock Springer Science \& Business Media, 2002.

\bibitem[Laurent \& Brecht(2018{\natexlab{a}})Laurent and
  Brecht]{laurent2017multilinear}
Thomas Laurent and James Brecht.
\newblock The multilinear structure of {ReLU} networks.
\newblock In \emph{International Conference on Machine Learning}, pp.\
  2914--2922, 2018{\natexlab{a}}.

\bibitem[Laurent \& Brecht(2018{\natexlab{b}})Laurent and
  Brecht]{laurent2018deep}
Thomas Laurent and James Brecht.
\newblock Deep linear networks with arbitrary loss: All local minima are
  global.
\newblock In \emph{International Conference on Machine Learning}, pp.\
  2908--2913, 2018{\natexlab{b}}.

\bibitem[Li \& Liang(2018)Li and Liang]{li2018learning}
Yuanzhi Li and Yingyu Liang.
\newblock Learning overparameterized neural networks via stochastic gradient
  descent on structured data.
\newblock In \emph{Advances in Neural Information Processing Systems}, pp.\
  8168--8177, 2018.

\bibitem[Li \& Yuan(2017)Li and Yuan]{li2017convergence}
Yuanzhi Li and Yang Yuan.
\newblock Convergence analysis of two-layer neural networks with {ReLU}
  activation.
\newblock In \emph{Advances in Neural Information Processing Systems}, pp.\
  597--607, 2017.

\bibitem[Liang et~al.(2018{\natexlab{a}})Liang, Sun, Lee, and
  Srikant]{liang2018adding}
Shiyu Liang, Ruoyu Sun, Jason~D Lee, and R~Srikant.
\newblock Adding one neuron can eliminate all bad local minima.
\newblock In \emph{Advances in Neural Information Processing Systems}, pp.\
  4355--4365, 2018{\natexlab{a}}.

\bibitem[Liang et~al.(2018{\natexlab{b}})Liang, Sun, Li, and
  Srikant]{liang2018understanding}
Shiyu Liang, Ruoyu Sun, Yixuan Li, and Rayadurgam Srikant.
\newblock Understanding the loss surface of neural networks for binary
  classification.
\newblock In \emph{International Conference on Machine Learning}, pp.\
  2840--2849, 2018{\natexlab{b}}.

\bibitem[Lu \& Kawaguchi(2017)Lu and Kawaguchi]{lu2017depth}
Haihao Lu and Kenji Kawaguchi.
\newblock Depth creates no bad local minima.
\newblock \emph{arXiv preprint arXiv:1702.08580}, 2017.

\bibitem[Nguyen \& Hein(2017)Nguyen and Hein]{nguyen2017loss}
Quynh Nguyen and Matthias Hein.
\newblock The loss surface of deep and wide neural networks.
\newblock In \emph{Proceedings of the 34th International Conference on Machine
  Learning}, volume~70, pp.\  2603--2612, 2017.

\bibitem[Nguyen \& Hein(2018)Nguyen and Hein]{nguyen2017losscnn}
Quynh Nguyen and Matthias Hein.
\newblock Optimization landscape and expressivity of deep {CNN}s.
\newblock In \emph{International Conference on Machine Learning}, pp.\
  3727--3736, 2018.

\bibitem[Safran \& Shamir(2018)Safran and Shamir]{safran2017spurious}
Itay Safran and Ohad Shamir.
\newblock Spurious local minima are common in two-layer {ReLU} neural networks.
\newblock In \emph{International Conference on Machine Learning}, pp.\
  4430--4438, 2018.

\bibitem[Shamir(2018)]{shamir2018resnets}
Ohad Shamir.
\newblock Are {ResNets} provably better than linear predictors?
\newblock In \emph{Advances in Neural Information Processing Systems}, pp.\
  505--514, 2018.

\bibitem[Soltanolkotabi(2017)]{soltanolkotabi2017learning}
Mahdi Soltanolkotabi.
\newblock Learning {ReLUs} via gradient descent.
\newblock In \emph{Advances in Neural Information Processing Systems}, pp.\
  2007--2017, 2017.

\bibitem[Soudry \& Carmon(2016)Soudry and Carmon]{soudry2016no}
Daniel Soudry and Yair Carmon.
\newblock No bad local minima: Data independent training error guarantees for
  multilayer neural networks.
\newblock \emph{arXiv preprint arXiv:1605.08361}, 2016.

\bibitem[Swirszcz et~al.(2016)Swirszcz, Czarnecki, and
  Pascanu]{swirszcz2016local}
Grzegorz Swirszcz, Wojciech~Marian Czarnecki, and Razvan Pascanu.
\newblock Local minima in training of neural networks.
\newblock \emph{arXiv preprint arXiv:1611.06310}, 2016.

\bibitem[Tian(2017)]{tian2017analytical}
Yuandong Tian.
\newblock An analytical formula of population gradient for two-layered {ReLU}
  network and its applications in convergence and critical point analysis.
\newblock In \emph{International Conference on Machine Learning}, pp.\
  3404--3413, 2017.

\bibitem[Venturi et~al.(2018)Venturi, Bandeira, and Bruna]{venturi2018neural}
Luca Venturi, Afonso Bandeira, and Joan Bruna.
\newblock Neural networks with finite intrinsic dimension have no spurious
  valleys.
\newblock \emph{arXiv preprint arXiv:1802.06384}, 2018.

\bibitem[Wang et~al.(2018)Wang, Giannakis, and Chen]{wang2018learning}
Gang Wang, Georgios~B Giannakis, and Jie Chen.
\newblock Learning {ReLU} networks on linearly separable data: Algorithm,
  optimality, and generalization.
\newblock \emph{arXiv preprint arXiv:1808.04685}, 2018.

\bibitem[Wu et~al.(2018)Wu, Luo, and Lee]{wu2018no}
Chenwei Wu, Jiajun Luo, and Jason~D Lee.
\newblock No spurious local minima in a two hidden unit {ReLU} network.
\newblock In \emph{International Conference on Learning Representations
  Workshop}, 2018.

\bibitem[Xie et~al.(2016)Xie, Liang, and Song]{xie2016diverse}
Bo~Xie, Yingyu Liang, and Le~Song.
\newblock Diverse neural network learns true target functions.
\newblock \emph{arXiv preprint arXiv:1611.03131}, 2016.

\bibitem[Yu \& Chen(1995)Yu and Chen]{yu1995local}
Xiao-Hu Yu and Guo-An Chen.
\newblock On the local minima free condition of backpropagation learning.
\newblock \emph{IEEE Transactions on Neural Networks}, 6\penalty0 (5):\penalty0
  1300--1303, 1995.

\bibitem[Yun et~al.(2018)Yun, Sra, and Jadbabaie]{yun2018global}
Chulhee Yun, Suvrit Sra, and Ali Jadbabaie.
\newblock Global optimality conditions for deep neural networks.
\newblock In \emph{International Conference on Learning Representations}, 2018.

\bibitem[Zhang et~al.(2018)Zhang, Yu, Wang, and Gu]{zhang2018learning}
Xiao Zhang, Yaodong Yu, Lingxiao Wang, and Quanquan Gu.
\newblock Learning one-hidden-layer {ReLU} networks via gradient descent.
\newblock \emph{arXiv preprint arXiv:1806.07808}, 2018.

\bibitem[Zhong et~al.(2017)Zhong, Song, Jain, Bartlett, and
  Dhillon]{zhong2017recovery}
Kai Zhong, Zhao Song, Prateek Jain, Peter~L Bartlett, and Inderjit~S Dhillon.
\newblock Recovery guarantees for one-hidden-layer neural networks.
\newblock In \emph{International Conference on Machine Learning}, pp.\
  4140--4149, 2017.

\bibitem[Zhou \& Liang(2018)Zhou and Liang]{zhou2018critical}
Yi~Zhou and Yingbin Liang.
\newblock Critical points of neural networks: Analytical forms and landscape
  properties.
\newblock In \emph{International Conference on Learning Representations}, 2018.

\bibitem[Zhou et~al.(2019)Zhou, Yang, Zhang, Liang, and Tarokh]{zhou2019sgd}
Yi~Zhou, Junjie Yang, Huishuai Zhang, Yingbin Liang, and Vahid Tarokh.
\newblock {SGD} converges to global minimum in deep learning via star-convex
  path.
\newblock In \emph{International Conference on Learning Representations}, 2019.

\bibitem[Zou et~al.(2018)Zou, Cao, Zhou, and Gu]{zou2018stochastic}
Difan Zou, Yuan Cao, Dongruo Zhou, and Quanquan Gu.
\newblock Stochastic gradient descent optimizes over-parameterized deep {ReLU}
  networks.
\newblock \emph{arXiv preprint arXiv:1811.08888}, 2018.

\end{thebibliography}
